\documentclass{article} %
\usepackage{iclr2027_conference,times}
\usepackage{graphicx}

\usepackage{amsmath,amsfonts,bm}

\def\eqref#1{equation~\ref{#1}}

\def\1{\bm{1}}

\DeclareMathAlphabet{\mathsfit}{\encodingdefault}{\sfdefault}{m}{sl}
\SetMathAlphabet{\mathsfit}{bold}{\encodingdefault}{\sfdefault}{bx}{n}

\newcommand{\Att}{A}
\newcommand{\im}{\mathrm{i}}
\newcommand{\euler}{\mathrm{e}}

\usepackage{hyperref}
\usepackage{url}
\usepackage{wrapfig}
\usepackage{adjustbox}
\usepackage{subfig}
\usepackage{booktabs}
\usepackage{amsthm}

\definecolor{gacolor}{rgb}{0.28,0.51,0.43}
\newif\ifdraft
\drafttrue

\newtheorem{theorem}{Theorem}[section]
\newtheorem{proposition}[theorem]{Proposition}
\newtheorem{lemma}[theorem]{Lemma}

\newtheorem{remark}[theorem]{Remark}

\draftfalse
\title{Pattern Formation in Transformers}

\author{Erkan Turan$^{1,*,\dagger}$ \quad
        Gaspard Abel$^{2,3,*}$ \quad
        Maks Ovsjanikov$^{1}$ \\
\normalfont $^{1}$LIX, \'Ecole Polytechnique, IP Paris, France \\
\normalfont $^{2}$Centre Borelli, ENS Paris-Saclay, Universit\'e Paris-Saclay, France \\
\normalfont $^{3}$Centre d'Analyse et de Math\'ematique Sociales, EHESS, CNRS, 75006 Paris, France
}

\iclrfinalcopy %
\begin{document}

\maketitle
\lhead{Preprint}
{\renewcommand{\thefootnote}{\fnsymbol{footnote}}%
\footnotetext[1]{Equal contribution.}%
\footnotetext[2]{Correspondence to: Erkan Turan \texttt{<turan@lix.polytechnique.fr>}}}

\begin{abstract}
What are the inductive biases of a Transformer architecture? 
Existing theory on how the forward pass shapes representations either considers whether Transformers escape from \emph{rank collapse} or demonstrates that self-attention drives tokens toward \emph{cluster} patterns. The latter view arises from an elegant dynamical systems perspective, but relies on simplified architectural assumptions, and does not explain the rich structures observed in practice. This leaves a major open question: when a full Transformer escapes rank collapse, how does it  structure token representations? Using pattern-formation theory, we show that the dynamical view of Transformers can account for Positional Encoding, Multi-Head Attention, and Output-Value geometry. We demonstrate that a full Transformer architecture imposes an inductive prior by selectively amplifying a rich set of previously unreported patterns, including traveling or rotating waves among others. We characterize the role of each architectural component in controlling which pattern is amplified, which ones stabilize, compete, or coexist. Finally, we show that these structures can act as a controllable \emph{dynamical prior} that  facilitates learning. By choosing both task-aligned positional encoding and weight initialization, we demonstrate improved data efficiency and accelerated optimization on controlled sequence tasks and with ConViT on CIFAR-10.

\end{abstract}

\section{Introduction}

Transformers are universal approximators~\citep{yun2020universal,dehghani2019universal} and Turing complete~\citep{perez2019turing}. These general results establish that Transformers can represent virtually any function \textit{in principle}, but they say very little about the structural biases of the architecture itself. Specifically, what representations does a forward pass through a Transformer architecture naturally promote before learning takes place? 

Existing theory that studies this question typically falls into one of two lines of work. The first focuses on \emph{rank collapse} (or oversmoothing), which analyzes how repeated self-attention drives all token representations toward a common vector~\citep{dong2021attention,noci2022signal}. The second, based on dynamical systems, characterizes the geometry of representations and shows that self-attention drives tokens into discrete \emph{clusters}~\citep{geshkovski2023clusters,bruno2026multiscale}. Yet, these elegant global results are derived under strong simplifications, omitting Positional Encoding (PE), Multi-Head Attention (MHA), and feed-forward nonlinearities (FFN), and do not account for the complex spatial and sequential patterns that Transformers capture in practice. This gap raises a fundamental question: when a full Transformer escapes collapse, what structural priors can the architecture impose on its representations? More practically, how do these patterns affect learning, and can they be controlled at initialization?

In this paper, we show that signals that survive rank collapse in Transformers are neither arbitrary noise, nor clusters alone. Using pattern-formation theory~\citep{cross1993pattern}, we demonstrate that around the collapsed state, a randomly initialized Transformer can exhibit complex emerging modes.
These joint token--feature modes are characterized by the Transformer's architectural components and are preferentially selected and amplified with depth, revealing a set of spontaneously emerging structures: clusters, standing or traveling waves, and feature rotations. We call the propensity toward these representations a \emph{dynamical prior}: a soft inductive bias present before learning and dictated by the architecture rather than by data. Recent evidence on the role of architectural biases in Transformers at initialization~\citep{zheng2025structured,li2026born} motivates the application of our approach to guide the training of these models.

Our contributions are threefold: 
(1) We derive a matrix-valued dispersion relation (a frequency-resolved gain matrix) for a Transformer block to determine the structure of emerging modes and show how each architectural component contributes to it. We catalog the representations it can amplify, including the known clustered states and previously unreported standing, traveling, feature-rotating and mixed ones (Sec.~\ref{sec:linear}).
(2) We derive amplitude equations that characterize how FFN and content-attention nonlinearities control saturation, phase drift and competition (Sec.~\ref{sec:amplitude}).
(3) We validate our theory, show how initializations that amplify task-relevant structure learn faster and are more data-efficient than gain-matched controls on controlled tasks and on CIFAR-10 with ConViT (Sec.~\ref{sec:experiments})~\footnote{Code to reproduce all experiments is available at \url{https://github.com/ehturan/pattern-formation-transformers}.}.

\section{Related Work}

\textbf{Signal propagation, rank collapse and oversmoothing.}
Self-attention in Transformers can rapidly collapse toward token-averaged representations~\citep{dong2021attention,kedia2024stable,naitsaada2025mind}. This hinders training, as it has been shown to induce vanishing gradients ~\citep{noci2022signal}.
This rank collapse phenomenon also plagues graph neural networks~\citep{oono2020graph}, BERT and ViTs~\citep{shi2022revisiting,dovonon2025oversmoothing}, where self-attention imposes a low-pass filter keeping only the homogeneous signal~\citep{wang2022antioversmoothing,park2022how}.
All of these ask \emph{whether} the signal survives, whereas we consider what \textit{shape it takes}.

\textbf{Transformers as dynamical systems.}
Treating tokens as interacting particles~\citep{lu2020multiparticle} has enabled further characterization of how clusters~\citep{geshkovski2023clusters}, multiscale dynamics~\citep{bruno2026multiscale}, and spectral selection~\citep{tomihari2026recurrent, kuehn2026spectral} occur in Transformers. These works provide elegant global results, but typically have to consider simplified dynamics, such as single head, symmetric or identity values, no PE and no FFN. In contrast, our local analysis retains all key architectural components and determines their role in the emergence of structures in Transformers.

\textbf{Initialization and inductive biases.} 
On the design side, we rely on the observation that initialization-target alignment is necessary for efficient learning~\citep{abbe2022initial}, and that random networks and Transformers carry systematic architecture-dependent biases~\citep{teney2024neural,li2026born,subramaniam2026training}. We make one such bias explicit and controllable by showing theoretically how ConViT~\citep{dascoli2021convit}, mimetic ~\citep{trockman2023mimetic} initializations and traveling waves~\citep{keller2023wave,keller2024traveling} impose dynamical priors that facilitate the training of these architectures.

\section{Setup: Information Flow through a Transformer at Initialization}
\label{sec:setup}
\paragraph{Transformer Architecture.}
\label{sec:model}
Let $X=(x_1,\ldots,x_N)\in\mathbb R^{N\times d}$ denote $N$ token representations $x_i\in\mathbb R^d$, with $h\in\{1,\ldots,H\}$ indexing attention heads. For head $h$, let
$Q_h,K_h,V_h\in\mathbb R^{d_h\times d}$ be the query, key and value matrices and let $b^{(h)}_{ij}$ denote the positional encoding. The attention weights are %
\begin{equation}
\Att_{ij}^{(h)}(X)
=
\frac{
\exp\!\left(d_h^{-1/2}\langle Q_hx_i,K_hx_j\rangle+b^{(h)}_{ij}\right)
}{
\sum_{m=1}^N
\exp\!\left(d_h^{-1/2}\langle Q_hx_i,K_hx_m\rangle+b^{(h)}_{im}\right)
}.
\label{eq:standard_attention}
\end{equation}
Writing the output projection as $O=[O_1\cdots O_H]$, define
$M_h:=O_hV_h\in\mathbb R^{d\times d}$, with $\operatorname{rank}(M_h)\le d_h$. The Multi-Head Attention is then expressed as: $\operatorname{MHA^{(\ell)}}(X) = \sum_h \sum_j\Att_{ij}^{(h)}(X)M_h^{(\ell)}x_j$.
Therefore, a Transformer block defines the mapping:
\begin{equation}
[\mathcal F^{(\ell)}(X)]_i
=
\mathcal H^{(\ell)}\!\left(x_i+\operatorname{MHA}_i^{(\ell)}(X)\right),
\qquad
\mathcal H^{(\ell)}(y)=y+\phi^{(\ell)}(y),
\label{eq:full_transformer_map}
\end{equation}
where $\phi^{(\ell)}$ is a Feed Forward Network (FFN). Repeated application of this block gives
$X^{(\ell+1)}=\mathcal F^{(\ell)}(X^{(\ell)})$.
We omit LayerNorm to streamline the analysis; it has been reported as not strictly necessary to prevent rank collapse ~\citep{noci2022signal} and can be replaced by point-wise nonlinearities~\citep{zhu2025transformers}. Pre-LN preserves the block structure of the linear theory but changes the base state (Remark~\ref{rem:manifold}, App.~\ref{app:layernorm}). %

A result motivating our analysis is that, with small residual scalings, this architecture can be viewed as a Lie--Trotter splitting~\citep{lu2020multiparticle} of the form:
\begin{equation}
\frac{dX}{dt}
=
\underbrace{\operatorname{MHA}(X)}_{\mathcal P(X):\,\text{propagation across tokens}}
+
\underbrace{\operatorname{FFN}(X)}_{\mathcal L(X):\,\text{local pointwise dynamics}}.
\label{eq:lie_trotter}
\end{equation}

Equation \ref{eq:lie_trotter} mirrors the classic \emph{reaction--diffusion} (or \emph{pattern-forming}) systems studied in dynamical systems, mathematical biology~\citep{turing1990chemical,cross1993pattern,cross2009pattern} (see App.~\ref{app:swift} for a brief primer). In such systems, nonlocal coupling across positions ($\mathcal{P}$) competes with local pointwise nonlinearities ($\mathcal{L}$). A key property of such systems is that when a spatially uniform state becomes unstable, perturbations do not grow arbitrarily; instead, the coupling selectively amplifies specific spatial frequencies, causing structured patterns.

\textbf{The collapsed state.}
In a Transformer, the spatially uniform state corresponds to \emph{rank collapse}~\citep{dong2021attention,noci2022signal,shi2022revisiting}: the set of homogeneous states $\mathcal M=\{\mathbf 1\otimes a:\ a\in\mathbb R^d\}$ in which all $N$ tokens share the exact same feature vector $a \in \mathbb{R}^d$. With uniform attention, without the residual and FFN, $\mathcal P$ acts as a pure averaging (diffusion) operator that contracts all tokens onto $\mathcal M$. While prior signal-propagation theory asks when $\mathcal M$ is attractive or avoided, we use the pattern-formation view to study what structures emerge when $\mathcal M$ becomes unstable.

\section{The Transformer Matrix Dispersion Relation}
\label{sec:linear}

Our goal in this section is to determine which structures a Transformer architecture preferentially amplifies at initialization. The starting point is to determine what can emerge from the homogeneous manifold $\mathcal{M}$. To do so, consider a small heterogeneous perturbation $U\in \mathbb R^{N \times d}$ around collapsed state $X_\star \in \mathcal{M}$. The following simplified Theorem determines the dynamics of this perturbation when it passes a Transformer block.

\begin{theorem}[Simplified Linear stability analysis]
\label{thm:stability}
For a single-head Transformer and with the aforementioned notations, the homogeneous manifold $\mathcal M$ is invariant under a Transformer block $\mathcal F$: $\mathcal F(\mathbf 1\otimes a) =\mathbf 1\otimes g(a)$, with $g(a):=\mathcal H(a+Va)$.

Let $X_\star=\mathbf 1\otimes a_\star\in\mathcal M$ with $g(a_\star)=a_\star$, and let
$A_\star=A(X_\star)$ be the attention matrix at $X_\star$.
In first order in $U$, token-level perturbations evolve through the Transformer map as:
\begin{equation}
    u_i^+ = C_\star\Big[u_i+\sum\nolimits_j (A_\star)_{ij} V u_j\Big],
\end{equation}
where $V$ is the value/output matrix and $C_\star$ is the FFN Jacobian at $a_\star$.

Further, let $\{ \lambda_q | q \in \{0,\dots N-1\}\}$ be the eigenvalues of $A_\star$ with corresponding eigenvectors $e_q$ (the index $q$ becomes a frequency in Sec.~\ref{sec:pe}).
Then, after a Transformer update, a perturbation $U = e_q \otimes v$ along $e_q$ evolves as:
\begin{equation}
    U^+ = e_q \otimes J(q)v,
    \qquad
    J(q) = C_\star(I_d+\lambda_q V).
    \label{eq:J_q}
\end{equation}

\end{theorem}

The full version of this Theorem is in App.~\ref{thm:stability_full}. It states that the eigenvectors (or modes) $e_q$ determine the available token-space structures, while the matrices $J(q)$ determine whether each structure grows ($\rho(J(q))>1$) or decays ($\rho(J(q))<1$) with depth, and along which feature-space directions (the eigenvectors of $J(q)$). A Transformer can therefore preferentially amplify particular joint token-feature structures even before learning: this is what we consider to be a \emph{dynamical prior}.
Note that a Jacobian analysis of weight-tied attention was conducted in~\citet{tomihari2026recurrent}, but without PE and FFN, and in which the token-space structure of $A_\star$ is left unresolved.

Inspired by pattern-formation theory, we call the matrix-valued map $q \rightarrow J(q)$ the \emph{Transformer Matrix Dispersion Relation}. We build it progressively, to isolate the roles of PE, $OV$ geometry, MHA, and the FFN, and finally use it to classify the structures promoted at linear order.

\subsection{Positional Encoding Lifts Heterogeneous Space Degeneracy}
\label{sec:pe}
We start with a single attention head without positional encoding and ask: do the linearized dynamics preferentially select a heterogeneous token-space structure?

\begin{proposition}[No PE: frequency degeneracy]
\label{prop:no_position}
Without PE $(b_{ij}=0)$, any homogeneous state $X_\star\in\mathcal M$ produces uniform attention, i.e. $\Att_\star=N^{-1}\mathbf1\mathbf1^\top$.
The token-space then decomposes into the two invariant eigenspaces $\mathbb R^N
= \operatorname{span}\{\mathbf 1\}
\oplus
\mathbf {1^\perp}$,
corresponding respectively to the homogeneous mode and the $(N-1)$-dimensional space of zero-mean heterogeneous modes. Their associated Jacobian matrices are
\begin{equation}
J(0)=C_\star(I+V),
\qquad
J(q)=C_\star
\quad(q\ne0).
\label{eq:no_pe_vector_dispersion}
\end{equation}
In linear order, all nonzero token frequencies have identical feature-space amplification.
\end{proposition}

This result demonstrates that, in linear order, the value matrix $V$ controls the preference between homogeneous and heterogeneous perturbations, consistent with the empirical observation that negative-identity $OV$ initialization provides improved performance in vision Transformers reported by~\cite{trockman2023mimetic}.
However, without positional information, the dispersion relation cannot favor one heterogeneous structure over another~\footnote{We will see in Th.~\ref{prop:cluster_amplitude} that degeneracy is lifted by the nonlinearities, forming clusters of token representations.}. We next show that PE provides precisely the missing mechanism.

\begin{proposition}[PE lifts frequency degeneracy]
\label{prop:pe_dispersion}
Suppose we have a relative PE, i.e. $b_{ij}=b_{r}$, where $r=i-j$, with periodic boundary conditions. At any homogeneous state $X_\star\in\mathcal M$, the attention matrix is circulant: $[A_\star]_{ij}=\Att_{i-j}$. Therefore, $\forall k \in \{0,\dots,N-1\}$, with $q =\frac{2\pi k}{N}$, $A_\star$ admits the $N$ Fourier eigenpairs $(\lambda(q),e_q)$, and $J(q)= C_\star(I_d+\lambda(q) V)$ as in Eq.~\ref{eq:J_q}:
\begin{equation}
    A_\star e_q=\lambda(q)e_q,
    \qquad
    (e_q)_j=\euler^{\im qj},
    \qquad
    \lambda(q)=\sum_r \Att_r \euler^{-\im qr}.
\end{equation}
\end{proposition}
This result gives positional encoding a direct dynamical interpretation: since $\lambda(q)$ is the Fourier transform of the attention matrix, designing the positional encoding is designing a filter in token-space, while $V$ and $C_\star$ determine how each transmitted mode evolves in feature space. The Transformer can now preferentially amplify specific token-space structures. This completes, the picture exhibited by \citep{trockman2023mimetic} where both negative $OV$ initialization and positional encodings appear crucial for efficient learning and improved performance in ViTs. Through our framework, we understand that negative $OV$ allows to favor heterogeneous features while positional encodings allows the architecture to discern these heterogeneities via the emergence of Fourier modes. An extension of this result without periodic boundary conditions is presented in App.~\ref{app:bulk}.

\subsection{Multi-head attention engineers the dispersion relation}
\label{sec:mha}

Distinguishing modes is not the same as being able to preferentially amplify a specific one. We therefore ask which modes a single attention head can select as the dominant growing mode. For a broad class of PEs, the answer is surprisingly restrictive.
\begin{proposition}[Single-head PE select spectral edges]
\label{prop:singlehead_spectral_edges}
Assume that $C_\star=cI$, $c>0$, $V=\eta I$, and let the positional encoding be reflection symmetric with real Fourier spectrum $\lambda(q)$ monotonically decreasing on $q\in[0,\pi]$. Then
\begin{equation}
    \rho(J(q))=c|1+\eta\lambda(q)|,
    \qquad
    \arg\max_{q\in[0,\pi]}\rho(J(q))
    \subseteq\{0,\pi\}.
\end{equation}
Hence, a single head can select either a low-frequency mode ($q_c=0$) or a high-frequency mode ($q_c=\pi$), but not an isolated intermediate wavelength. This extends to any $V$ (App.~\ref{app:linear_proofs}).
\end{proposition}

This result refines the usual low-pass picture of self-attention~\citep{wang2022antioversmoothing,park2022how}: the sign of the $OV$ gain decides whether a single head is low-pass ($\eta>0$, $q_c=0$) or high-pass ($-2/(1+\lambda(\pi))<\eta<0$, $q_c=\pi$), but never band-pass. The result covers reflection-symmetric localized positional encodings such as bidirectional ALiBi~\citep{press2022train} and the centered Gaussian kernel of ConViT~\citep{dascoli2021convit}. Isolated intermediate wavelengths therefore require positional encodings outside this class, such as causal ALiBi and off-center ConViT, or with RoPE~\citep{su2024roformer} that induces a non-symmetric, non-monotone positional profile on $\mathcal M$ (more details in App.~\ref{app:rope}). We next show that multi-head attention provides another solution:  different heads contribute distinct positional spectra and $OV$ geometries, which combine into a matrix-valued filter capable of producing a richer dispersion relation.

\begin{figure}[t]
\centering
\includegraphics[width=0.8\textwidth]{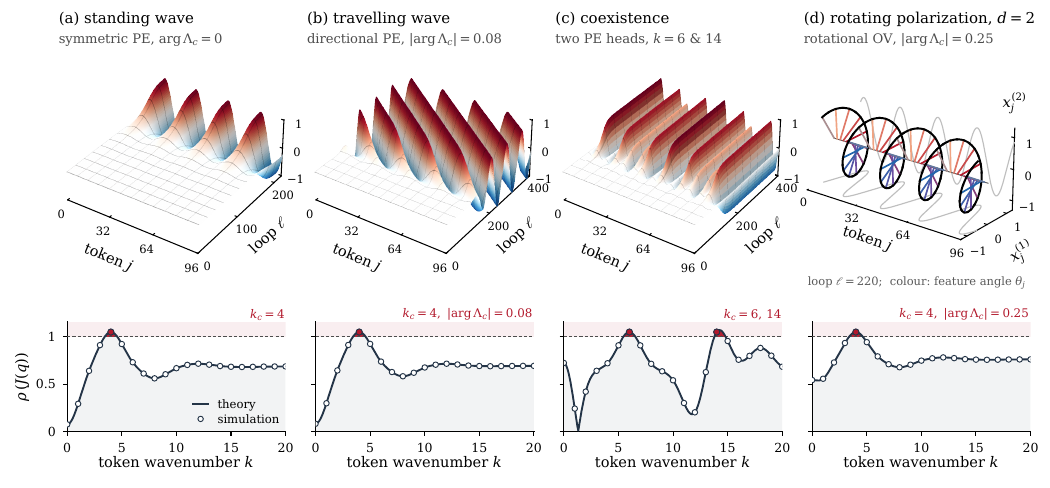}
\caption{\textbf{Taxonomy of patterns in Transformers:} $2D$ visualization of the propagation of the accessible modes, with their corresponding dispersion relation. }
\label{fig:taxo_patterns}
\end{figure}

\begin{theorem}[Transformer Matrix Dispersion Relation]
    \label{thm:TMDR}
    Suppose the PE is translation invariant, i.e. $b_{i,j} = b_{i-j}$, let $X_\star\in\mathcal M$, and let $M_h := O_h V_h$ for a concatenated MHA of $H$ heads. For a perturbation $U=e_q\otimes v$ along Fourier mode $q$, token-wise perturbations evolve as
    \begin{equation}
        u_i^+ =\euler^{\im qi}J(q)v, \qquad J(q)=C_\star\left[I+\mathcal D(q)\right], \qquad \mathcal D(q):=\sum\nolimits_{h=1}^{H}\lambda_h(q)M_h.
        \label{eq:TMDR}
    \end{equation}
    Thus, MHA constructs a matrix-valued spectral filter $\mathcal D(q)$, rather than a simple token--feature coupling.
\end{theorem}

Eq.~\ref{eq:TMDR} is the central object of this work: at initialization, a Transformer block acts as a matrix-valued spectral filter on token-space modes. For each head, the positional encoding determines the scalar response $\lambda_h(q)$, the $OV$ matrices $M_h$ the feature directions through which that response acts, and the FFN Jacobian $C_\star$ rescales the result.

\textbf{Scope and weight tying.} The results hold on all of $\mathcal M$ and along depth, also for layer-dependent blocks (App.~\ref{app:untied}). Weight tying is thus not required for the linear theory; we introduce it only in Sec.~\ref{sec:amplitude}, where depth as repeated iteration of the same map allows us to study the nonlinear stabilization of growing modes, as in~\cite{geshkovski2023clusters, bruno2026multiscale}. %

\subsection{A catalogue of joint token--feature patterns}

Now that the Transformer Matrix Dispersion Relation is laid out, we can turn its eigenpairs into concrete representations. The following proposition makes the growth of a joint token-feature mode explicit and shows how the phase of $\Lambda_c$ and the geometry of $v_c$ determine the resulting morphologies.

\begin{proposition}[Shape of the emerging patterns]
    \label{prop:catalogue}
    Let $a_\star$ be a fixed point of $g$ and $(\Lambda_c,v_c)$ a simple eigenpair of $J(q_c)$ with $|\Lambda_c|>1$. The perturbation $U = e_{q_c}\otimes v_c$ along mode $q_c$ grows exponentially with depth and is of the form
    \begin{equation}
        u_j^{(\ell)}
        \sim
        B_0\Lambda_c^\ell v_c \euler^{\im q_cj}
        +
        \overline{
        B_0\Lambda_c^\ell v_c \euler^{\im q_cj}
        },
    \label{eq:joint_mode}
    \end{equation}
    with $\Lambda_c=\euler^{\gamma_c+i\omega_c}$, $B_0 = |B_0| \euler^{\im \theta_0}$ the initial perturbation amplitude. Writing $v_c=v_\Re+\im v_\Im$, Eq.~\ref{eq:joint_mode} gives
     \begin{equation}
     u_j^{(\ell)}
     \simeq
     2|B_0| \euler^{\ell\gamma_c}
     \left[
     v_\Re\cos\Theta_{j\ell}
     -
     v_\Im\sin\Theta_{j\ell}
     \right],
     \qquad
     \Theta_{j\ell}
     =
     q_cj+\ell\omega_c+\theta_0.
     \label{eq:vector_wave_form}
     \end{equation}
\end{proposition}

Eq.~\ref{eq:vector_wave_form} gives a direct dictionary between the spectral data of $J(q_c)$ and the resulting representation. The wavenumber $q_c$ sets the spatial scale across tokens, the growth rate $\gamma_c$ controls amplification with depth, the phase $\omega_c$ determines whether the pattern drifts across layers, and the real and imaginary parts of $v_c$ describe its geometry in feature space. Different combinations therefore give rise to different morphologies, simulated in Fig.~\ref{fig:taxo_patterns}: standing waves ($\omega_c=0$, $v_c \in \mathbb R^d$), traveling waves ($\omega_c \neq 0$, from a directional PE), feature rotations ($v_c \in \mathbb C^d$, from non-symmetric $OV$ geometry), and their mixtures. Notably, traveling waves have been \emph{hard-}designed into neural architectures~\citep{keller2023wave,keller2024traveling} for sequence learning. We demonstrate here that they emerge spontaneously in Transformers whenever their PE is directional or their $OV$ non-symmetric.

\section{Past Linear Order: Amplitude Equations for Transformers}
\label{sec:amplitude}
The dispersion relation in Sec. \ref{sec:linear} answers a first question: \emph{which joint token-feature structures does a Transformer initially amplify?} Linear growth, however, cannot be the end of the story. Once the selected structure becomes appreciable, the approximation underlying $J(q)$ breaks down.
This leaves three questions that are invisible to the dispersion relation: does a linearly growing mode stabilize at a finite amplitude? When several structures have exactly the same linear gain, which one is eventually realized? When distinct modes grow simultaneously, do they coexist or suppress one another? The non-linear analysis below addresses these questions with the second tool of pattern-formation theory, the \emph{amplitude equation}~\citep{cross1993pattern}.  

\subsection{From linear growth to finite-amplitude patterns}
\begin{figure}[t]
\centering
\includegraphics[width=0.8\textwidth]{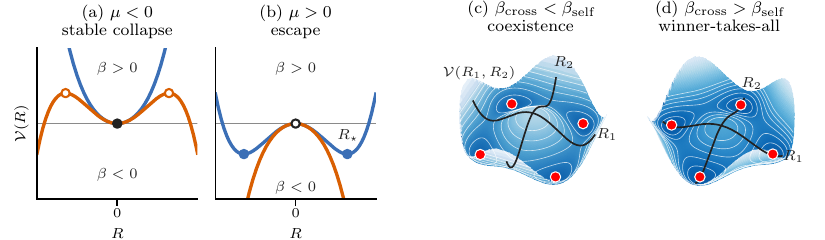}
\caption{\textbf{Depth evolution descends an effective potential.} (a,b) One critical mode, Eq.~\ref{eq:potential_1d} (blue: $\beta>0$, orange: $\beta<0$; disks: stable states, circles: unstable): the dispersion relation sets $\mu$, the nonlinearities set $\beta$. (c,d) Two critical modes, Eq.~\ref{eq:two_mode_potential}, over the signed amplitudes $(R_1,R_2)$: the ratio of cross-suppression $\beta_{\rm cross}$ to self-saturation $\beta_{\rm self}$ decides where the wells are.}
\label{fig:potential}
\end{figure}

Consider first a single critical mode $q_c$ whose gain is close to the instability threshold, $\rho(J(q_c))=1+\mu$ with $\vert \mu\vert \ll 1$. Near this threshold, the selected mode evolves slowly while modes whose gains remain strictly below one decay rapidly, so the token dynamics can be reduced to an analysis of amplitude of the critical modes. For a single stationary pattern with amplitude $B$, the reduced dynamics has the generic form
\begin{equation}
B^+-B=\mu B-\beta B^3=-\mathcal V'(B),
\qquad
\mathcal V(B)=-\frac{\mu}{2}B^2+\frac{\beta}{4}B^4,
\qquad
\mu=\rho\big(J(q_c)\big)-1 .
\label{eq:potential_1d}
\end{equation}
The two coefficients have different meanings (Fig.~\ref{fig:potential}). The linear coefficient $\mu$ is already determined by the dispersion relation of Sec.~\ref{sec:linear}: it sets the curvature of the effective potential around the collapsed state, and therefore when that state loses stability. The nonlinear coefficient $\beta$, by contrast, depends on terms discarded by the linearization and determines what replaces it. If $\mu>0$ and $\beta>0$, nonlinear feedback counteracts linear instability and the pattern saturates at a finite amplitude $B_\star=\sqrt{\mu/|\beta|}$. If instead $\beta<0$, the cubic nonlinearity reinforces the instability.

\subsection{Amplitude Equations of Clusters and Travelling Waves}

We now study the effects of these higher-order terms on two representations: clusters and traveling waves. For clarity, we work with the scalar reduction $d=1$ of the full Transformer block in Eq.~\ref{eq:full_transformer_map} (the general $d>1$ case reduce to the same form by projection onto the left eigenvector), around $a_\star=0$ (the extension to $a_\star\neq0$ is sketched in App.~\ref{app:offorigin}) with $OV$ gain $\eta$, content coupling $\chi$, and translation-invariant PE $b_{i-j}$:
\begin{equation}
A_{ij}(x)=\frac{\euler^{\chi x_ix_j+b_{i-j}}}{\sum_m \euler^{\chi x_ix_m+b_{i-m}}},\qquad H(x)=x+\eta A(x)x,\qquad F(x)=H(x)+\nu\phi(H(x)),
\label{eq:scalar_map}
\end{equation}
Nonlinearity $\phi$ is smooth and $\phi(0)=0, \ C_\star=\alpha_1:=1+\nu\phi'(0),\ \alpha_2:=\nu\phi''(0)/2$ and $\alpha_3:=\nu\phi'''(0)/6$, where $\nu$ is the residual scale.
$H$ and $F$ are the residual and Transformer mapping of Eq.~\ref{eq:full_transformer_map}, respectively. Expanding $F$ around $x=0$ (Lemma~\ref{lem:scalar_cubic_expansion}, App.~\ref{app:cubic}) separates the dispersion relation from the higher-order terms:
\begin{equation}
F(x)=
\underbrace{\textstyle\sum_q J(q)\,\hat x_q\,e_q}_{\text{dispersion relation}}
+\underbrace{\alpha_2(Lx)^{\circ2}+\alpha_3(Lx)^{\circ3}}_{\text{FFN nonlinearity}}
+\underbrace{\alpha_1\eta\chi\,\mathcal C(x)}_{\text{content attention}}
+O(\|x\|^4),
\label{eq:scalar_dispersion_notation}
\end{equation}
where $x=\sum_q\hat x_qe_q$ and $J(q)=\alpha_1(1+\eta\lambda_q)$ is the dispersion relation of Theorem~\ref{thm:stability} with $V=\eta$ and $\lambda_q$ the eigenvalues of $A_\star=A(0)$. The FFN acts on the attention output $Lx := (I+\eta A_\star)x$, and content-dependent attention is $\mathcal C(x):=x\circ[A_\star x^{\circ2}-(A_\star x)^{\circ2}]$ ($\circ$: component-wise product).

We expose a division of labor between Transformer components. We have shown that PE, $OV$ geometry, and the linearized FFN determine \textbf{which mode is initially amplified} through $J(q)$, i.e., $\mu$. We demonstrate here how content-dependent attention and higher order properties of FFN determine, through $\beta$, \textbf{what happens to that mode once it has grown}.
\begin{theorem}[Clusters, simplified]
\label{prop:cluster_amplitude}
Without PE and with a negative $OV$ gain, $\eta\in(-2,0)$, all heterogeneous modes share the amplification $J(q\neq0)=1+\mu$ (Proposition~\ref{prop:no_position}). Near threshold, the separation $u$ of two balanced clusters, $x_i=\pm u$ up to a small common shift, obeys
\begin{equation}
u^+-u=\mu u-\beta_{\rm cl}u^3,\qquad
\beta_{\rm cl}=-\eta\chi-\alpha_3+\frac{2(1+\eta)}{\eta}\alpha_2^2.
\label{eq:cluster_amplitude_general}
\end{equation}
If $\mu>0$ and $\beta_{\rm cl}>0$, the clusters saturate at a finite separation $u_\star=\sqrt{\mu/\beta_{\rm cl}}$, which is stable along the cluster direction. Against intra-cluster perturbations, it is stable only for an odd activation ($\alpha_2=0$) with $\alpha_3<0$. The full statement and proof are in App.~\ref{app:cluster}.
\end{theorem}
\textbf{Non-linear effects on clustering.} Here the instability is driven by the FFN ($C_\star=\alpha_1>1$), with $\eta<0$ stabilizing the mean, whereas the clusters of~\citet{geshkovski2023clusters} are driven by attention with $V=I$. Which structure emerges from the degenerate sector $\mathbf 1^\perp$ is decided at cubic order, by ingredients absent from the dispersion relation. Negative-OV content attention, $\eta\chi<0$, is what holds a cluster: as it tightens, its members attend more to one another. Interestingly, an odd saturating activation ($\alpha_3<0$) adds damping, whereas any even part $\alpha_2\neq0$ makes one cluster spread, so GELU-type FFNs tend to break balanced clusters that $\tanh$-type FFNs keep. Consequently, activations also bias the emerging structure, which to our knowledge was not noted before. This activation function induced structural bias is also present in message-passing GNNs ~\citep{turan2026beyond}.

\begin{theorem}[Traveling waves, simplified]
\label{prop:travelling_amplitude}
With a directional PE, the critical mode has a complex gain $J(q_c)=(1+\mu)\euler^{\im\omega_c}$
: at linear order, the wave grows by $1+\mu$ and advances by a phase $\omega_c$ per layer. Near threshold, its amplitude $G$, measured in the frame moving with the wave, obeys
\begin{equation}
G^+-G=\mu G-\beta_{\rm tr}|G|^2G,
\label{eq:travelling_amplitude_main}
\end{equation}
with a complex coefficient $\beta_{\rm tr}$. If $\mu>0$ and $\Re(\beta_{\rm tr})>0$, the wave saturates at a stable amplitude $R_\star=|G_\star|$ and propagates at phase speed $\Omega$ per layer:
\begin{equation}
R_\star^2=\frac{\mu}{\Re(\beta_{\rm tr})},\qquad
\Omega=\omega_c-\Im(\beta_{\rm tr})R_\star^2.
\label{eq:nonlinear_wave_speed}
\end{equation}
Thus the real part of $\beta_{\rm tr}$ controls saturation, while its imaginary part gives the nonlinear correction to the propagation speed. The full statement and proof are in App.~\ref{app:travelling}.
\end{theorem}
As illustrated in Fig.~\ref{fig:taxo_patterns}(d), a complex eigenvector of $J(q_c)$ adds a feature-plane rotation; the same reduction applies after projection onto that eigenmode.

\textbf{Scaling laws and design levers.} All cases (clusters, traveling waves, and standing waves in App.~\ref{app:standing}) predict amplitude scaling laws invisible at linear order: past threshold, cluster separation and wave amplitude grow as $\sqrt{\mu}$, and the wave speed shifts linearly in $\mu$ (Fig.~\ref{fig:phase_structure}a,b). This section allowed us to refine the propagation--collapse dichotomy~\citep{giorlandino2026failure,cowsikGeometricDynamicsSignal2025} into a refined family of stable structured representations as seen in Fig. \ref{fig:phase_structure}c,d.

\subsection{Competition: winner-takes-all versus coexistence}
\label{sec:competition}
What happens when two modes become critical at once? The linear theory allows both to grow indistinguishably; the non-linear analysis predicts a competition: stable coexistence or winner-takes-all. 
\label{sec:potential}

\begin{theorem}[Two-mode competition, simplified]
\label{prop:mode_competition}
Let two standing modes $q_1,q_2$ cross threshold together with the same gain $1+\mu$, $\mu>0$. Their amplitudes obey
\begin{equation}
B_i^+-B_i=\big(\mu-\beta_{\rm self}|B_i|^2-\beta_{\rm cross}|B_j|^2\big)B_i,\qquad i\neq j\in\{1,2\},
\label{eq:two_mode_amplitudes_main}
\end{equation}
where $\beta_{\rm self}>0$ measures how each mode saturates itself and $\beta_{\rm cross}$ how strongly it suppresses the other. Equivalently, the moduli $R_i=|B_i|$ descend the effective potential
\begin{equation}
R_i^+-R_i=-\frac{\partial\mathcal V}{\partial R_i},
\qquad
\mathcal V(R_1,R_2)=-\frac{\mu}{2}\big(R_1^2+R_2^2\big)+\frac{\beta_{\rm self}}{4}\big(R_1^4+R_2^4\big)+\frac{\beta_{\rm cross}}{2}R_1^2R_2^2,
\label{eq:two_mode_potential}
\end{equation}
If $\beta_{\rm cross}<\beta_{\rm self}$, the mixed state $R_1^2=R_2^2=\mu/(\beta_{\rm self}+\beta_{\rm cross})$ is the stable minimum and the two modes coexist; if $\beta_{\rm cross}>\beta_{\rm self}$, it becomes a saddle, the single-mode states $R_i^2=\mu/\beta_{\rm self}$ are the stable minima, and the initial condition selects the winner (Fig.~\ref{fig:potential}c,d). The full statement, for unequal gains, and the proof are in App.~\ref{app:competition}.
\end{theorem}
Kernels and heads that transmit the mixed harmonics $q_1 \pm q_2$ alter the cross-suppression through their stable-mode resolvents, and suppressing those harmonics weakens this route of competition.
\begin{figure}[t]
\centering
\includegraphics[width=0.8\textwidth]{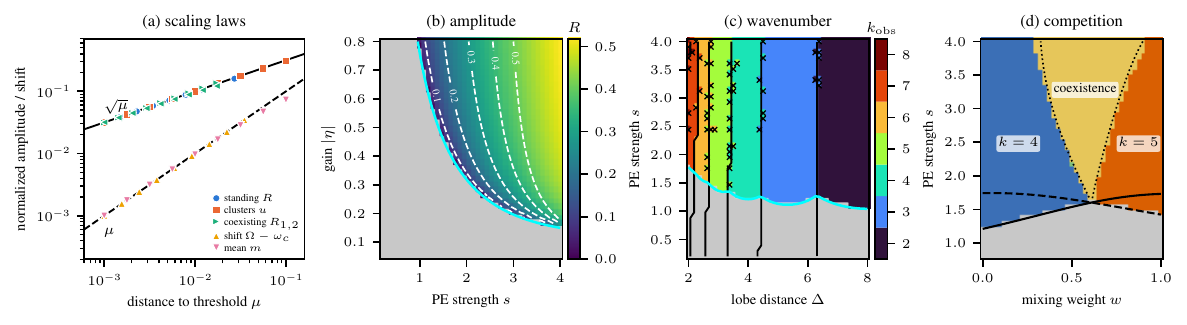}
\caption{\textbf{Prediction of the amplitudes of representations}. Lines: theory, colors: simulation, gray: collapsed state, cyan line: threshold $\max_q|J(q)|=1$. (a): Amplitude scaling laws for multiple patterns versus $\mu$. (b): Amplitude heatmap of the traveling wave, in the $(|\eta|,s)$ map; white dashed: predicted $R_\star$. (c) Critical mode $k_\mathrm{obs}$ in the $(\Delta,s)$ map; black: theoretically predicted mode transitions ($\times$: nonlinear selection differs); (d): Competition and coexistence of modes $4$ (solid), $5$ (dashed) in $(w,s)$ with winner-takes-all exchange lines of Thm.~\ref{prop:mode_competition} (dotted).}
\label{fig:phase_structure}
\end{figure}

\section{Experiments}
We first validate the theoretical results of Secs.~\ref{sec:linear} and \ref{sec:amplitude}, and then provide practical evidence that accessible representations can induce dynamical priors at initialization on stylized experiments.
\label{sec:experiments}
\subsection{The theory is quantitatively predictive}
\label{sec:validation}
First, we validate the theory of Sec.~\ref{sec:amplitude} by iterating the scalar map in Eq.~\ref{eq:scalar_map} on a ring of $N=32$ tokens, with a Transformer using $\phi = \tanh$, $OV$ gain $|\eta|$ and gaussian PE $b_r=sK_\Delta(r)$, where $s$ is the PE strength and $\Delta$ the distance between two peaks (details in App.~\ref{app:fig3}). We measure the dominant Fourier modes, their amplitudes, and, for directional kernels, their phase advance.

The scaling laws of Fig.~\ref{fig:phase_structure}(a) validate our theoretical results on the amplitude growth of multiple representations close to criticality.
Panel (b) confirms our predictions in Eq.~\ref{eq:nonlinear_wave_speed} on how the wave amplitude $R$ increases with gain $|\eta|$ and PE strength $s$. Panel (c) varies distance $\Delta$ between peaks (lobes) of the PE and tests whether the selected wavelength follows the theory $k_\mathrm{th} = \frac{N}{2\pi}\arg\max_q|J(q)|$. Finally, (d) tests two-mode competition for modes $k=4$ and $k=5$, when varying the preference for one mode or the other with mixing $w$. Our theory clearly distinguishes coexistence from winner-takes-all selection. Together, these results show that the theory predicts not only when collapse is avoided, but also the structure and dynamics of the pattern that replaces it.

\subsection{Engineering the prior on controlled tasks}
\label{sec:controlled_experiments}
Next, we test whether the dynamical prior we claim to provide can be aligned with a task, and if such an alignment helps the learning process, as stated in~\citep{abbe2022initial}. Each of the three tasks (Fig.~\ref{fig:inductive_prior_results}, left) targets one of the identified structures and probes its attribute: \textbf{T1} reconstructs periodic sequences of mode $q_0$, favoring $q_c=q_0$. \textbf{T2} translates a structured sequence by a prescribed shift, probing the phase $\theta$ of the mode. \textbf{T3} reconstructs under a prescribed feature-space rotation $\psi$, probing the eigenvector geometry of $J(q)$. More details are in App.~\ref{app:structured}. For each task, using identical architectures, parameter counts and gains, we initialize $4$ models that differ only in \emph{which} token--feature mode they amplify. We then compute the AUC of their learning curves of each $(\mathrm{task,init})$ tuple and report their $z$-scores.

The heatmaps of Fig.~\ref{fig:inductive_prior_results} (right) clearly indicate that task-aligned initializations benefit from a strong dynamical prior. It originates from engineering the architecture so that the critical mode at initialization is aligned with the problem at hand.
Supplementary experiments shown in Fig.~\ref{fig:efficiency} show that this practical gain occurs most markedly early in training.

\begin{figure}[t]
\centering
\includegraphics[width=0.8\textwidth]{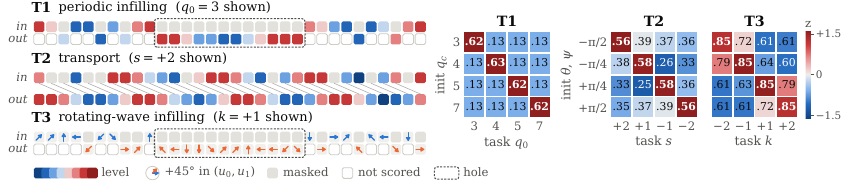}
\caption{\textbf{Task--initialization alignment.} Left: the periodic (T1), transport (T2) and rotation (T3) tasks (perturbed, partially masked targets). Right:  AUC for each (task, initialization) pair; the color corresponds to the $z$-score. Higher scores in the diagonal support the dynamical prior induced by mode criticality.}
\label{fig:inductive_prior_results}
\end{figure}

\begin{figure}[t]
\centering
\subfloat[]{\includegraphics[width=0.64\textwidth]{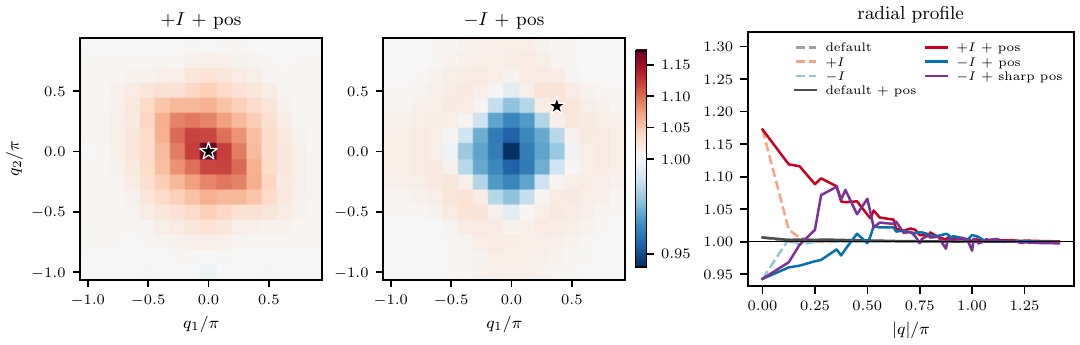}}%
\hspace{0.5em}
\subfloat[]{\raisebox{2em}{%

\begin{adjustbox}{max width = 0.28\textwidth}
\footnotesize\setlength{\tabcolsep}{2pt}%
\begin{tabular}[b]{lccc}
\toprule
Init. & Final Acc. & AUC $\uparrow$ & $t_{90}\downarrow$ \\
\midrule
default + pos
        & $89.6 \pm 0.4$
        & $78.0 \pm 0.3$
        & $92.0\ (1/3)$ \\
$+I$ + pos
        & $90.2 \pm 0.2$
        & $79.3 \pm 0.3$
        & $85.3 \pm 3.2$ \\
$-I$ + pos
        & $91.1 \pm 0.3$
        & $80.8 \pm 0.4$
        & $74.3 \pm 3.1$ \\
$-I$ + sharp pos
        & $\mathbf{91.5 \pm 0.2}$
        & $\mathbf{81.3 \pm 0.3}$
        & $\mathbf{69.0 \pm 1.7}$\\
\bottomrule
\end{tabular}\end{adjustbox}}}
\caption{\textbf{Dispersion-aware initialization of ConViT on CIFAR-10.} (a) Heatmap of mode gains and radial profile of the dispersion relation for the initializations variants (colours: mode stability; $\star$: leading mode). (b) Final Accuracy, AUC and $t_{90}$; all variants use localized PE, ``sharp'': sharpened kernel.}
\label{fig:dispersion_vision}
\end{figure}

\subsection{Dispersion-aware initialization of ConViT on CIFAR-10}
\label{sec:vision}
Finally, we train a weight-tied ConViT on CIFAR-10 (more details in App.~\ref{app:vision}) for classification and show that engineering mode amplification facilitates learning in practical architectures. On the left of Fig.~\ref{fig:dispersion_vision}(a), we plot the dispersion relation for various parameterizations: localized and sharpened PE~\citep{dascoli2021convit} variants shift amplification toward heterogeneous spatial modes. Reporting classification results in Fig.~\ref{fig:dispersion_vision}(b) shows that combining $-I$ with localized PE substantially accelerates optimization; sharpening the dispersion structure helps further ($91.5\%$ test accuracy). This explains why mimetic initialization's $OV = -I$~\citep{trockman2023mimetic} and ConViT's localized kernels~\citep{dascoli2021convit} help: together they move the leading multiplier of $J(\mathbf q)$ off $\mathbf q=0$. In further experiments we track the dispersion relation during training (Fig.~\ref{fig:app_vision_evolution}) which shows that heterogeneous patterns are favored. 

\section{Conclusion}
\label{sec:conclusion}
In this paper, we show that a randomly initialized Transformer is not structureless once it escapes rank collapse: its forward pass multiplies each token frequency by a matrix-valued gain $J(q)$, set by PE, $OV$ geometry and the FFN, that amplifies specific joint token--feature structures such as traveling and rotative waves. This dynamical prior refines the binary propagate-or-collapse picture into a mode-resolved one. The nonlinear theory says which structures saturate, drift or coexist, and aligning the prior with the task accelerates optimization and improves data efficiency.

\subsubsection*{Acknowledgments}
We would like to thank Julien Gaubil and Antoine Gu\'edon for valuable discussions. Parts of this work were supported by the ERC Consolidator Grant 101087347 (VEGA), as well as gifts from Ansys Inc.\ and Adobe Research.

\label{end:main}\typeout{ENDMAIN page=\thepage\space filled=\the\pagetotal\space of \the\vsize}

\subsection*{AI use statement}
We used a large language model and coding assistant for two tasks: (i) editing the
manuscript, i.e. tightening and restructuring prose and (ii) refining and formatting plots and figures. We
have reviewed all AI-assisted work and take full responsibility for the content
of this paper.

\bibliography{iclr2027_conference}
\bibliographystyle{iclr2027_conference}
\appendix

\section{Proofs for the matrix-valued linear theory}
\label{app:linear_proofs}

This section proves the linear results of Sec.~\ref{sec:linear}. Token states are
columns $x_i\in\mathbb R^d$, and a perturbation is denoted
$U=(u_1,\ldots,u_N)$. For a token-space vector $s\in\mathbb C^N$ and a
feature vector $v\in\mathbb C^d$, we write $s\otimes v$ for the perturbation
whose $i$th token is $s_i v$.

\subsection{Full version Theorem~\ref{thm:stability} and proof}

\begin{theorem}[Linear stability and token–feature block decomposition]
Let $\mathcal M:=\{\mathbf 1\otimes a:\ a\in\mathbb R^d\}$ denote the homogeneous states. $\mathcal M$ is invariant under any Transformer block: $\mathcal F(\mathbf 1\otimes a)=\mathbf 1\otimes g(a)$ with $g(a):=\mathcal H(a+Va)$ for a single head ($V\to\sum_hM_h$ for MHA). Let $X_\star=\mathbf 1\otimes a_\star\in\mathcal M$ and $U^+ := \mathcal F (X_\star + U) - \mathcal F (X_\star)$ the evolution of perturbation $U\in \mathbb R^N\otimes\mathbb R^d$ through a Transformer block. Then, to first order in $U$, the token-wise perturbation $u_i^+$ evolves as:
\begin{equation}
    u_i^+
    =
    C_\star\left[
        u_i+\sum_{j=1}^N(A_\star)_{ij}Vu_j
    \right],
\end{equation}
where $A_\star=A(X_\star)$ and $C_\star=D\mathcal H(y_\star)=I_d+D\phi(a_\star+Va_\star)$ is the Jacobian of the residual FFN branch at $y_\star=a_\star+Va_\star$ (the $\lambda=1$ block is $Dg(a_\star)=C_\star(I+V)$). Equivalently, the full linearized Jacobian is $\mathcal J_\star=(I_N\otimes C_\star) \left[I_N\otimes I_d+A_\star\otimes V\right].$ If $A_\star=S\Lambda S^{-1}$ is diagonalizable, with
$\Lambda=\operatorname{diag}(\lambda_1,\ldots,\lambda_N)$, then
\begin{equation}
    (S^{-1}\otimes I_d)\mathcal J_\star(S\otimes I_d)
    =
    \bigoplus_{k=1}^N J(\lambda_k),
    \qquad
    J(\lambda)=C_\star\left[I_d+\lambda V\right],
\end{equation}
and $\operatorname{sp}(\mathcal J_\star)=\bigcup_k\operatorname{sp}(J(\lambda_k))$ holds even without diagonalizability (Schur form). Hence, the transverse dynamics decompose into independent token-space blocks indexed by the eigenmodes of $A_\star$. If $a_\star$ is a fixed point of $g$, then $X_\star$ is a linearly stable fixed point of $\mathcal F$ if and only if
\begin{equation}
    \max_{\lambda\in\mathrm{sp}(A_\star)}
    \rho\!\left(J(\lambda)\right)<1.
    \label{eq:matrix_stability}
\end{equation}
Along an orbit $a_\ell=g^\ell(a_0)$ on $\mathcal M$, the transverse dynamics is the cocycle $U\mapsto\mathcal J(a_\ell)U$, which splits into the products $\prod_{\ell<L}J_\ell(\lambda_{k,\ell})$, $J_\ell(\lambda)=C(a_\ell)[I_d+\lambda V]$, whenever the $A(\mathbf 1 \otimes a_\ell)$ share an eigenbasis.
(Secs.~\ref{sec:pe}--\ref{sec:mha}).
\label{thm:stability_full}
\end{theorem}

Its proof is below.

\begin{proof}
Write the effective one-head attention output as
\begin{equation}
    \mathcal A_i(X)
    :=\sum_{j=1}^N A_{ij}(X)Vx_j,
\end{equation}
where, in the one-head setting, any output projection can be absorbed into
$V$. Let
\begin{equation}
    y_i(X):=x_i+\mathcal A_i(X),
    \qquad
    [\mathcal F(X)]_i=\mathcal H(y_i(X)).
\end{equation}
Every row of $A(X)$ sums to one for every $X$:
\begin{equation}
    A(X)\mathbf1=\mathbf1.
\end{equation}
Consequently, at a homogeneous state $X=\mathbf1\otimes a$ all value vectors coincide and $\mathcal A_i(\mathbf1\otimes a)=Va$ for every $i$, whatever the positional bias or mask; since $\mathcal H$ is token-wise, $\mathcal F(\mathbf1\otimes a)=\mathbf1\otimes g(a)$ with $g(a)=\mathcal H(a+Va)$, which proves the invariance of $\mathcal M$. Now fix $X_\star=\mathbf1\otimes a_\star\in\mathcal M$, not necessarily a fixed point. Differentiating the row-sum identity at $X_\star$ in an arbitrary direction $U$ gives
\begin{equation}
    DA_{X_\star}[U]\mathbf1=0.
    \label{eq:app_rowstoch_derivative}
\end{equation}

We now differentiate the attention output. By the product rule,
\begin{align}
D\mathcal A_i(X_\star)[U]
&=
\sum_j\big(DA_{X_\star}[U]\big)_{ij}Va_\star
+\sum_j(A_\star)_{ij}Vu_j \\
&=
\left[\sum_j\big(DA_{X_\star}[U]\big)_{ij}\right]Va_\star
+\sum_j(A_\star)_{ij}Vu_j \\
&=
\sum_j(A_\star)_{ij}Vu_j,
\label{eq:app_singlehead_attention_derivative}
\end{align}
where the first term vanishes by Eq.~\eqref{eq:app_rowstoch_derivative}.
Thus first-order changes in the attention weights do not act on the
homogeneous value vector.

Let $y_\star=a_\star+Va_\star$ denote the common value of $y_i(X_\star)$ and set
\begin{equation}
    C_\star:=D\mathcal H(y_\star)=I+D\phi(y_\star).
\end{equation}
Since $\mathcal H$ acts token-wise, the chain rule and
Eq.~\eqref{eq:app_singlehead_attention_derivative} yield
\begin{equation}
    u_i^+
    =
    C_\star\left[
        u_i+\sum_{j=1}^N(A_\star)_{ij}Vu_j
    \right],
\end{equation}
which is the first claim.

Viewing $U\in\mathbb R^N\otimes\mathbb R^d$, this is equivalently
\begin{equation}
    \mathcal J_\star
    =
    (I_N\otimes C_\star)
    \left[
        I_N\otimes I_d+A_\star\otimes V
    \right].
\end{equation}
If $A_\star=S\Lambda S^{-1}$ with
$\Lambda=\operatorname{diag}(\lambda_1,\ldots,\lambda_N)$, then
\begin{align}
&(S^{-1}\otimes I_d)\mathcal J_\star(S\otimes I_d) \\
&\qquad=
(I_N\otimes C_\star)
\left[
I_N\otimes I_d+\Lambda\otimes V
\right]
=
\bigoplus_{k=1}^N
C_\star(I_d+\lambda_kV).
\end{align}
Hence the token eigenmodes of $A_\star$ define invariant $d$-dimensional
feature blocks
\begin{equation}
    J(\lambda_k)=C_\star(I_d+\lambda_kV).
\end{equation}
If $A_\star$ is not diagonalizable, take a Schur decomposition
$A_\star=STS^{*}$ with $T$ upper triangular:
$(S^{*}\otimes I_d)\mathcal J_\star(S\otimes I_d)=(I_N\otimes C_\star)[I_N\otimes I_d+T\otimes V]$
is block upper triangular with diagonal blocks $J(T_{kk})$, so
$\operatorname{sp}(\mathcal J_\star)=\bigcup_k\operatorname{sp}(J(\lambda_k))$
in all cases. Therefore
\begin{equation}
    \rho(\mathcal J_\star)
    =
    \max_{\lambda\in\operatorname{sp}(A_\star)}
    \rho\!\left(J(\lambda)\right).
\end{equation}
When $a_\star$ is a fixed point of $g$, $X_\star$ is a fixed point of
$\mathcal F$, and for a discrete-time map asymptotic linear stability is
equivalent to all multipliers of the Jacobian lying strictly inside the unit
circle; the stability criterion is precisely Eq.~\eqref{eq:matrix_stability}.
Along an orbit $a_\ell=g^\ell(a_0)$ the same computation at each
$X_{a_\ell}$ gives
$U^{(\ell+1)}=\mathcal J(a_\ell)U^{(\ell)}+O(\|U^{(\ell)}\|^2)$ with
$\mathcal J(a_\ell)=(I_N\otimes C(a_\ell))[I_N\otimes I_d+A(a_\ell)\otimes V]$;
if the matrices $A(a_\ell)$ share the eigenbasis $S$, conjugating by
$S\otimes I_d$ block-diagonalizes every factor simultaneously and the
$L$-step propagator is $\bigoplus_k\prod_{\ell<L}J_\ell(\lambda_{k,\ell})$,
whose growth rate is the transverse Lyapunov exponent stated in the theorem.
\end{proof}

\subsection{Proof of Proposition~\ref{prop:no_position}}

\begin{proof}
With one head and $b_{ij}=0$, the content logit is identical for every
pair $(i,j)$ at a homogeneous state, because all queries are equal and all
keys are equal. Each row of the softmax is therefore uniform:
\begin{equation}
    A_\star=\frac1N\mathbf1\mathbf1^\top.
\end{equation}
The vector $\mathbf1$ is an eigenvector with eigenvalue $1$, whereas every
$x\in\mathbf1^\perp$ satisfies $A_\star x=0$. Thus
\begin{equation}
    \mathbb R^N
    =
    \operatorname{span}\{\mathbf1\}
    \oplus
    \mathbf1^\perp
\end{equation}
is the required invariant decomposition. Substituting the two token-space
eigenvalues $\lambda=1$ and $\lambda=0$ into
Theorem~\ref{thm:stability} gives
\begin{equation}
    J(0)=C_\star(I+V),
    \qquad
    J(q)=C_\star
    \quad(q\neq0),
\end{equation}
which is Eq.~\eqref{eq:no_pe_vector_dispersion}. Hence every nonzero token
frequency has the same feature-space multipliers.
\end{proof}

\subsection{Proof of Proposition~\ref{prop:pe_dispersion}}

\begin{proof}
Assume periodic relative positional logits with
$b_{ij}=b_{i-j}$. At a homogeneous state the content part of the one-head
logit is independent of $i,j$, and therefore cancels between numerator and
denominator of the row softmax. Consequently
\begin{equation}
    (A_\star)_{ij}=\Att_{i-j},
    \qquad
    \Att_r=\frac{e^{b_r}}{\sum_s e^{b_s}},
\end{equation}
so $A_\star$ is circulant.

For the Fourier vector $(e_q)_j=e^{\im qj}$ and the convention
$r=i-j$, we obtain
\begin{align}
    (A_\star e_q)_i
    &=
    \sum_j \Att_{i-j}e^{\im qj}\\
    &=
    e^{\im qi}\sum_r \Att_re^{-\im qr}.
\end{align}
Thus
\begin{equation}
    A_\star e_q=\lambda(q)e_q,
    \qquad
    \lambda(q)=\sum_r \Att_re^{-\im qr}.
\end{equation}
On the $N$-point periodic grid the distinct frequencies are
$q=2\pi k/N$, $k=0,\ldots,N-1$.

Applying Theorem~\ref{thm:stability} to the token eigenvalue
$\lambda(q)$ gives, for $U=e_q\otimes v$,
\begin{equation}
    U^+=e_q\otimes J(q)v,
    \qquad
    J(q)=C_\star[I+\lambda(q)V],
\end{equation}
Since the discrete Fourier transform is
invertible, the full Jacobian is similar over $\mathbb C$ to
$\operatorname{diag}_qJ(q)$ and therefore
\begin{equation}
    \rho(\mathcal J_\star)=\max_q\rho(J(q)).
\end{equation}
This gives the stability condition of Theorem~\ref{thm:stability_full}. If $\lambda(q)$ varies with
$q$, then different token frequencies see different feature-space blocks,
so the heterogeneous frequency degeneracy is lifted.
\end{proof}

\subsection{Proof of Proposition~\ref{prop:singlehead_spectral_edges}}

\begin{proof}
Under the scalar specialization $C_\star=cI$, $V=\eta I$,
\begin{equation}
    J(q)=c[1+\eta\lambda(q)]I,
\end{equation}
and hence
\begin{equation}
    \rho(J(q))=c|1+\eta\lambda(q)|.
\end{equation}
Because $\lambda(q)$ is real and monotone on $[0,\pi]$, its image is the
interval with endpoints $\lambda(0)$ and $\lambda(\pi)$. Define
\begin{equation}
    g(x):=|1+\eta x|.
\end{equation}
The function $g$ is convex. Therefore its maximum over any compact interval
is attained at an endpoint, which gives
\begin{equation}
    \max_{q\in[0,\pi]}\rho(J(q))
    =
    c\max\left\{
       |1+\eta\lambda(0)|,\,
       |1+\eta\lambda(\pi)|
    \right\}.
    \label{eq:app_singlehead_edge_max}
\end{equation}
Thus an isolated interior frequency cannot be a strict maximizer. If
$\eta\neq0$ and $\lambda(q)$ is strictly monotone, no interior point can
tie an endpoint maximum, and
\begin{equation}
    \arg\max_{q\in[0,\pi]}\rho(J(q))
    \subseteq\{0,\pi\},
\end{equation}
as stated in the non-degenerate case. For a general $V$ with $C_\star=cI$, the eigenvalues of $J(q)$ are $c(1+\nu\lambda(q))$, $\nu\in\mathrm{sp}(V)$, so $\rho(J(q))=c\max_\nu|1+\nu\lambda(q)|$ is a maximum of convex functions of $\lambda(q)$, hence convex, and the same endpoint argument applies.
\end{proof}

\subsection{Proof of Theorem~\ref{thm:TMDR}}

\begin{proof}
For head $h$, write the post-output-projection contribution as
\begin{equation}
    \mathcal A_i^{(h)}(X)
    :=
    \sum_{j=1}^N A_{ij}^{(h)}(X)M_hx_j,
    \qquad
    M_h=O_hV_h.
\end{equation}
The full pre-FFN residual state is
\begin{equation}
    y_i(X)
    =
    x_i+\sum_{h=1}^H\mathcal A_i^{(h)}(X).
\end{equation}
As in the proof of Theorem~\ref{thm:stability},
row-stochasticity gives
\begin{equation}
    DA_{X_\star}^{(h)}[U]\mathbf1=0.
\end{equation}
Since the value vector $M_ha_\star$ is the same at every token, the
derivative of the $h$th head is therefore
\begin{equation}
D\mathcal A_i^{(h)}(X_\star)[U]
=
\sum_{j=1}^N(A_{h,\star})_{ij}M_hu_j.
\end{equation}
With $C_\star=D\mathcal H(y_\star)$, the full linearization is
\begin{equation}
    u_i^+
    =
    C_\star\left[
        u_i+\sum_{h=1}^H\sum_{j=1}^N
        (A_{h,\star})_{ij}M_hu_j
    \right].
    \label{eq:app_full_vector_linearization}
\end{equation}

For translation-invariant relative positional logits, the content logit of
each head is constant at $X_\star$ and cancels from the row softmax.
Therefore every $A_{h,\star}$ is circulant. Write
\begin{equation}
    A_{h,\star}e_q=\lambda_h(q)e_q.
\end{equation}
For the separated perturbation $U=e_q\otimes v$,
Eq.~\eqref{eq:app_full_vector_linearization} becomes
\begin{align}
    u_i^+
    &=
    e^{\im qi}
    C_\star
    \left[
        I+\sum_{h=1}^H\lambda_h(q)M_h
    \right]v.
\end{align}
Thus the feature-space block in Fourier sector $q$ is
\begin{equation}
    J(q)
    =
    C_\star
    \left[
        I+\sum_{h=1}^H\lambda_h(q)M_h
    \right]
    =
    C_\star[I+\mathcal D(q)],
\end{equation}
which is Eq.~\eqref{eq:TMDR}. Since the Fourier transform block-diagonalizes
all circulant $A_{h,\star}$ simultaneously, the full Jacobian is similar to
$\operatorname{diag}_qJ(q)$ and
\begin{equation}
    \rho(\mathcal J_\star)=\max_q\rho(J(q)).
\end{equation}
Hence a homogeneous fixed point is linearly stable exactly when
$\max_q\rho(J(q))<1$. Along an orbit $a_\ell$ on $\mathcal M$ every
$A_h(a_\ell)$ is circulant, since translation invariance holds at each
homogeneous state, so the Fourier basis diagonalizes all layers
simultaneously and the $L$-layer propagator of mode $q$ is
$\prod_{\ell<L}J_\ell(q)$ with
$J_\ell(q)=C(a_\ell)[I+\sum_h\lambda_h(q;a_\ell)M_h]$.
\end{proof}

\subsection{Proof of Proposition~\ref{prop:catalogue}}

\begin{proof}
Let $v_c$ be an eigenvector of the critical block:
\begin{equation}
    J(q_c)v_c=\Lambda_cv_c.
\end{equation}
Then $e_{q_c}\otimes v_c$ is an eigenvector of the full complexified
Jacobian. Because the Transformer parameters are real,
\begin{equation}
    J(-q_c)=\overline{J(q_c)},
\end{equation}
so $e_{-q_c}\otimes\overline{v_c}$ is the conjugate mode with multiplier
$\overline{\Lambda_c}$.

A real perturbation in the corresponding two-dimensional real invariant
subspace can be written
\begin{equation}
    u_j^{(0)}
    =
    B_0v_ce^{\im q_cj}
    +
    \overline{B_0v_ce^{\im q_cj}}.
\end{equation}
After $\ell$ applications of the linearized block,
\begin{equation}
    u_j^{(\ell)}
    =
    B_0\Lambda_c^\ell v_ce^{\im q_cj}
    +
    \overline{
    B_0\Lambda_c^\ell v_ce^{\im q_cj}},
\end{equation}
which is Eq.~\eqref{eq:joint_mode}.

Now write
\begin{equation}
    \Lambda_c=e^{\gamma_c+\im\omega_c},
    \qquad
    B_0=|B_0|e^{\im\theta_0},
    \qquad
    v_c=a+\im b.
\end{equation}
Taking the real part explicitly gives
\begin{equation}
    u_j^{(\ell)}
    =
    2|B_0|e^{\ell\gamma_c}
    \left[
        a\cos\Theta_{j\ell}
        -
        b\sin\Theta_{j\ell}
    \right],
    \qquad
    \Theta_{j\ell}
    =
    q_cj+\ell\omega_c+\theta_0,
\end{equation}
which is Eq.~\eqref{eq:vector_wave_form} with the amplitude written
explicitly as $|B_0|$. The standing, traveling, feature-rotating, and
mixed cases are obtained by specializing $q_c$, $\omega_c$, and the real
and imaginary parts of $v_c$.
\end{proof}

\subsection{Open boundaries: bulk equivalence}
\label{app:bulk}

Propositions~\ref{prop:pe_dispersion}--\ref{prop:catalogue} use periodic
boundary conditions. The following statement makes precise the sense in
which they describe the bulk of a sequence with open boundaries.

\begin{proposition}[Bulk equivalence]
\label{prop:bulk}
Let the positional kernel have finite range $\xi$, i.e.\ $\Att_r=0$ for
$|r|>\xi$, and let $\mathcal F^{\rm o}$ and $\mathcal F^{\rm p}$ denote the
block with open and periodic boundary conditions on $N$ tokens. Then for
every state $X$ and every $i$ with $\xi<i\le N-\xi$,
$[\mathcal F^{\rm o}(X)]_i=[\mathcal F^{\rm p}(X)]_i$. Consequently, for
every $X$ and every $L\ge1$,
\begin{equation}
    [(\mathcal F^{\rm o})^L(X)]_i=[(\mathcal F^{\rm p})^L(X)]_i
    \qquad\text{for all } L\xi<i\le N-L\xi,
\end{equation}
and the same holds for the linearizations $\mathcal J^{\rm o}$,
$\mathcal J^{\rm p}$ at any homogeneous state.
\end{proposition}

\begin{proof}
Row $i$ of the attention involves only the tokens $j$ with
$\Att_{i-j}\neq0$, i.e.\ $|i-j|\le\xi$. For $\xi<i\le N-\xi$ these are the
same tokens in both settings, with the same content and positional logits
and hence the same softmax normalizer, and no wrap-around index occurs.
Since the FFN is token-wise,
$[\mathcal F^{\rm o}(X)]_i=[\mathcal F^{\rm p}(X)]_i$ for such $i$. For
$L\ge2$, $[\mathcal F^{\rm o}(X^{\rm o})]_i$ with $L\xi<i\le N-L\xi$
depends only on the tokens $X^{\rm o}_j$ with $|i-j|\le\xi$, all of which
satisfy $(L-1)\xi<j\le N-(L-1)\xi$ and therefore coincide with
$X^{\rm p}_j$ by induction. The statement for the linearizations follows
by differentiating.
\end{proof}

The dispersion relation is therefore a bulk property, exact on a core of
$N-2L\xi$ tokens. For kernels with exponential tails (ALiBi, Gaussian),
truncating at range $\xi$ changes each attention row by at most the tail
mass $\delta(\xi)$ in total variation, and the $L$-step discrepancy is
$O\big(L\,\delta(\xi)\max_q\rho(J(q))^L\big)$, negligible once $\xi$ is a
few decay lengths. Two remarks. First, for reflection-symmetric kernels the
open-boundary attention matrix is $A^{\rm o}=D^{-1}T$ with $T$ symmetric
Toeplitz and $D$ the diagonal of row normalizers; it is similar to the
symmetric matrix $D^{-1/2}TD^{-1/2}$, so its spectrum is real and close to
the circulant symbol for localized kernels. For directional kernels
$A^{\rm o}$ is a non-normal Toeplitz matrix whose eigenvalues do not
converge to the circulant symbol~\citep{trefethen2005spectra}: the periodic
dispersion relation governs finite-depth growth but not the $L\to\infty$
spectrum, which is the distinction between convective and absolute
instability in open flows~\citep{chomaz2005global}. Second, the argument
applies verbatim to causal attention with a local kernel: the rows $i>\xi$
are those of the circulant built from the one-sided kernel
$\Att_r\mathbb 1_{r\ge0}$, whose symbol is complex, so causal masking enters
the theory as an extreme directional kernel.

\subsection{RoPE on the homogeneous manifold}
\label{app:rope}
\begin{remark}[Homogeneous manifold, RoPE, LayerNorm and untied weights]
\label{rem:manifold}
(i) Theorems~\ref{thm:stability} and~\ref{thm:TMDR} hold at every point of $\mathcal M$: along an orbit only $C(a_\ell)$ (and, for RoPE, the kernel) varies, and finite-depth growth is governed by $\prod_\ell J_\ell(q)$, which is what Sec.~\ref{app:vision_growth_validation} measures. (ii) RoPE is not an additive bias, but at $X_a$ its logit is $a^\top Q^\top R_{j-i}Ka$: a translation-invariant, generally non-symmetric kernel of size $O(\|a\|^2)$, so Theorem~\ref{thm:TMDR} applies with an $a$-dependent symbol $\lambda(q;a)$, invisible at $a=0$. (iii) LayerNorm is pointwise: with pre-LN the block structure is preserved with $\mathcal D(q)\to\mathcal D(q)\,D\mathrm{LN}(a)$ and $C_\star=I+D\phi(\mathrm{LN}(y_\star))\,D\mathrm{LN}(y_\star)$, which requires $a\neq0$ (App.~\ref{app:layernorm}).
\end{remark}

Rotary encoding replaces the additive bias by position-dependent rotations
of queries and keys: the logit is
$d_h^{-1/2}\langle R_iQx_i,R_jKx_j\rangle=d_h^{-1/2}x_i^\top Q^\top R_{j-i}Kx_j$,
where $R_r=\bigoplus_kR(\theta_kr)$ is block-diagonal with $2\times2$
rotations at frequencies $\theta_1,\dots,\theta_{d_h/2}$ and
$R_i^\top R_j=R_{j-i}$. At a homogeneous state $X_a$ the logit equals
\begin{equation}
    b_r(a):=d_h^{-1/2}\,a^\top Q^\top R_{r}Ka
    =d_h^{-1/2}\sum_{k}\Big[\langle p_k,k_k\rangle\cos(\theta_kr)
      +(p_k\times k_k)\sin(\theta_kr)\Big],
    \qquad r=j-i,
\end{equation}
where $p_k,k_k\in\mathbb R^2$ are the $k$th $2$-blocks of $Qa$ and $Ka$
and $p\times k=p_1k_2-p_2k_1$. Hence $b_r(a)$ is translation invariant,
the attention at $X_a$ is circulant (on the ring, for
$\theta_k\in\frac{2\pi}N\mathbb Z$; otherwise up to the boundary effects of
App.~\ref{app:bulk}), and Theorem~\ref{thm:TMDR} applies with the
$a$-dependent symbol $\lambda(q;a)$, the DFT of
$\Att_r(a)\propto\euler^{b_r(a)}$; the proof is unchanged since the values
$Vx_j$ are not rotated. Three properties follow. (i) $b_r(a)=O(\|a\|^2)$:
at $a=0$ RoPE yields uniform attention and no wavelength selection, which
is why the base state must be taken on $\mathcal M$ rather than at the
origin. (ii) $b_r(a)$ is reflection symmetric only if all $p_k\times k_k$
vanish; generically it is directional, $\lambda(q;a)$ is complex, and RoPE
supports traveling modes with a single head. (iii) $b_r(a)$ is not
monotone: if one block dominates,
$\Att_r(a)\propto\euler^{\rho\cos(\theta_kr-\varphi)}$ is a periodic (von
Mises) kernel of period $2\pi/\theta_k$, whose DFT is supported on the
multiples of $\theta_k$ with $\lambda(\pm\theta_k;a)=\euler^{\mp\im\varphi}I_1(\rho)/I_0(\rho)$
dominating the non-zero frequencies ($I_n$ the modified Bessel functions);
for instance, for $\eta<0$ and $\varphi=\pi$ the leading mode of
$|L_q|=|1+\eta\lambda(q;a)|$ is the intermediate wavelength
$q_c=\pm\theta_k$, set jointly by the rotary frequency and the content
$a$. Proposition~\ref{prop:singlehead_spectral_edges} does not apply
because its monotonicity hypothesis fails.

\subsection{LayerNorm}
\label{app:layernorm}

Consider the pre-LN block
$[\mathcal F(X)]_i=\mathcal H\big(x_i+\sum_h\sum_jA^{(h)}_{ij}(\mathrm{LN}(X))M_h\mathrm{LN}(x_j)\big)$,
$\mathcal H(y)=y+\phi(\mathrm{LN}(y))$, with
$\mathrm{LN}(x)=\gamma\circ z(x)+\beta$,
$z(x)=(x-\mu(x)\mathbf 1_d)/\sigma(x)$. Since LN is token-wise,
$\mathrm{LN}(X_a)$ is homogeneous, so at $X_a$ the values
$M_h\mathrm{LN}(a)$ coincide across tokens and the content logits are
constant: the two facts used in the proofs of
Theorems~\ref{thm:stability} and~\ref{thm:TMDR}. Repeating those proofs
with the chain rule gives, for $\sigma(a)>0$,
\begin{align}
    J(q)&=C_\star\Big[I+\sum_h\lambda_h(q)M_hN_1\Big],
    \qquad
    N_1=D\mathrm{LN}(a),\nonumber\\
    C_\star&=I+D\phi(\mathrm{LN}(y_\star))\,D\mathrm{LN}(y_\star),
    \qquad
    y_\star=a+\sum_hM_h\mathrm{LN}(a),
\end{align}
with
\begin{equation}
    D\mathrm{LN}(x)=\frac{1}{\sigma(x)}\operatorname{diag}(\gamma)
    \Big[P-\frac{z(x)z(x)^\top}{d}\Big],
    \qquad P=I-\tfrac1d\mathbf 1_d\mathbf 1_d^\top,
\end{equation}
where $P-zz^\top/d$ is the orthogonal projector onto
$\{\mathbf 1_d,z(x)\}^\perp$. Thus LayerNorm preserves the Fourier block
structure and only inserts the rank-$(d-2)$ factor $N_1$: it removes the
mean and scale directions from the feature dynamics and rescales the
coupling by $1/\sigma(a)$. Two consequences follow. First, LN is singular
at $a=0$, so the base state must be a point $a\neq0$ of $\mathcal M$.
Second, with pre-LN the increment
$g(a)-a=\sum_hM_h\mathrm{LN}(a)+\phi(\mathrm{LN}(y(a)))$ is bounded and,
for large $\|a\|$, depends on $a$ only through the direction of $Pa$; a
fixed point of $g$ is therefore not guaranteed, the residual norm
typically grows along depth, and the transverse coupling decays like
$1/\sigma(a_\ell)$, which flattens the dispersion relation with depth
(consistent with the reduced effect of positional sharpening under pre-LN
in Table~\ref{tab:app_vision_ablations}). Post-LN and normalization-free
pointwise substitutes such as DyT~\citep{zhu2025transformers} renormalize
the state, so fixed points on $\mathcal M$ exist and
Sec.~\ref{sec:amplitude} applies, with the Taylor coefficients of the
normalization contributing to $\alpha_2,\alpha_3$.

\subsection{Layer-dependent weights}
\label{app:untied}

Nothing in the proofs of Theorems~\ref{thm:stability} and~\ref{thm:TMDR}
uses that the block is the same at every layer. With layer-dependent
parameters $(Q_h^{(\ell)},K_h^{(\ell)},M_h^{(\ell)},\phi^{(\ell)})$ and
translation-invariant PE, each layer is circulant at any homogeneous
state, the Fourier basis diagonalizes all layers simultaneously, and the
$L$-layer transverse propagator of mode $q$ is $\prod_{\ell<L}J_\ell(q)$
with $J_\ell(q)=C^{(\ell)}(a_\ell)[I+\sum_h\lambda^{(\ell)}_h(q)M_h^{(\ell)}]$,
exactly as along an orbit of the tied block. Weight tying is used only to
speak of fixed points in Sec.~\ref{sec:amplitude}.

\section{Cubic expansion and minimal nonlinear closure}
\label{app:cubic}
All scalar nonlinear results in Sec.~\ref{sec:amplitude} start from the same local expansion of Eq.~\eqref{eq:scalar_map}. We derive it once, then use the linear sectors of Sec.~\ref{sec:linear} to determine which additional modes must be retained in each case. This makes the reduced ansatz a consequence of nonlinear mode generation rather than an independent assumption.

For the remainder estimates below we assume $\phi$ is $C^5$ in a neighborhood of the origin; the displayed coefficients depend only on its first three derivatives. Let $\Att_0:=A(0)$ and
\begin{equation}
L:=I+\eta\Att_0,
\qquad
J(q)=\alpha_1L_q,
\qquad
L_q:=1+\eta\lambda_q,
\end{equation}
where the $\lambda_q$ are the eigenvalues of $\Att_0$. For vectors, $x^{\circ m}$ denotes component-wise powers and $u\circ v$ the Hadamard product; $\Pi_q y$ denotes the Fourier coefficient of $y$ in sector $q$.

\begin{lemma}[Third-order expansion of the scalar Transformer map]
\label{lem:scalar_cubic_expansion}
For the scalar map in Eq.~\ref{eq:scalar_map} around $x=0$,
\begin{equation}
F(x)=\alpha_1Lx+\alpha_2(Lx)^{\circ2}+\alpha_3(Lx)^{\circ3}
      +\alpha_1\eta\chi\,\mathcal C(x)+O(\|x\|^4),
\label{eq:app_full_cubic_expansion}
\end{equation}
where, writing $x=\sum_q\hat x_qe_q$, the first-order term $\alpha_1Lx=\sum_q J(q)\,\hat x_q\,e_q$ is the linear map of Theorem~\ref{thm:stability}, $C_\star(I+A_\star V)$, with $C_\star=\alpha_1$, $V=\eta$ and $A_\star=\Att_0$; it acts on each Fourier mode through the scalar dispersion relation $J(q)$. The content-attention term is
\begin{equation}
\mathcal C(x):=x\circ\left[\Att_0(x^{\circ2})-(\Att_0x)^{\circ2}\right].
\label{eq:app_C_def}
\end{equation}
In particular, the state dependence of attention first enters at cubic order at the zero state.
\end{lemma}
\begin{proof}
Row by row,
\begin{equation}
A_{ij}(x)=\frac{(\Att_0)_{ij}\euler^{\chi x_ix_j}}
                 {\sum_m(\Att_0)_{im}\euler^{\chi x_ix_m}}.
\end{equation}
Since $\chi x_ix_j=O(\|x\|^2)$, expansion of the exponential and row normalizer gives
\begin{equation}
A_{ij}(x)=(\Att_0)_{ij}\left[1+\chi x_ix_j
-\chi x_i\sum_m(\Att_0)_{im}x_m+O(\|x\|^4)\right].
\end{equation}
Multiplying by $x_j$ and summing,
\begin{equation}
A(x)x=\Att_0x+\chi\mathcal C(x)+O(\|x\|^5).
\label{eq:app_attention_cubic_scalar}
\end{equation}
Hence $H(x)=Lx+\eta\chi\mathcal C(x)+O(\|x\|^5)$. Taylor expansion of $F(x)=H(x)+\nu\phi(H(x))$ gives Eq.~\ref{eq:app_full_cubic_expansion}. The content logit is quadratic at $x=0$, so no attention-dependent quadratic term appears.
\end{proof}

\subsection{From linear sectors to the minimal nonlinear ansatz}
\label{app:center_manifold}
At linear order the Fourier sectors are independent: mode $q$ is multiplied by $J(q)$. Nonlinear products couple them by wave-vector addition. For a simple critical pair $\pm q_c$, write only its linear contribution as
\begin{equation}
x_c=B\,e_{q_c}+\overline B\,e_{-q_c}.
\end{equation}
The quadratic term in Lemma~\ref{lem:scalar_cubic_expansion} generates only
\begin{equation}
0,\quad \pm2q_c,
\end{equation}
whereas the direct cubic terms generate $\pm q_c$ and $\pm3q_c$. We call the pair \emph{non-resonant} if $4q_c\not\equiv0 \pmod{2\pi}$, so that $0$, $\pm2q_c$ and $\pm3q_c$ are distinct from $\pm q_c$ (this excludes $q_c=\pi$), and if $|J(0)|,|J(2q_c)|<1$. Under the non-resonance assumptions used below, the mean $m$ and second harmonic are linearly stable. To obtain the equation for $B$ through cubic order, it is therefore sufficient to retain their $O(|B|^2)$ response:
\begin{equation}
x_j=B\euler^{\im q_cj}+\overline B\euler^{-\im q_cj}
    +m+Z\euler^{2\im q_cj}+\overline Z\euler^{-2\im q_cj}
    +O(|B|^3),
\label{eq:minimal_one_mode_ansatz}
\end{equation}
where $Z$ is the quadratic amplitude factor. The third harmonic is indeed generated at $O(|B|^3)$, but its feedback into the critical sector occurs only beyond cubic order, so it does not need to be included in the minimal closure.

\section{Clusters: proof of Theorem~\ref{prop:cluster_amplitude}}
\label{app:cluster}
We first state the full version of Theorem~\ref{prop:cluster_amplitude}.
\begin{theorem}[Clusters, full version of Theorem~\ref{prop:cluster_amplitude}]
\label{thm:cluster_full}
Suppose there is no PE, $N$ is even and $\eta\in(-2,0)$, so that the homogeneous direction is stable, while the heterogeneous sector $\mathbf1^\perp$ loses stability through $J(q\neq0)=\alpha_1=1+\mu$. For a balanced two-cluster pattern $x=u\,s+m\mathbf1$, $s_i\in\{\pm1\}$ and $\sum_i s_i=0$, the only quadratic mode generated by $us$ is the homogeneous one. It is bound\footnote{\cite{haken1996slaving} refers to this phenomenon as ``slaving''.} to the cluster amplitude as
\begin{equation}
    m=-\frac{\alpha_2}{\eta}u^2+O(\mu u^2,u^4),
\end{equation}
and
\begin{equation}
u^+-u=\mu u-\beta_{\rm cl}u^3+O(\mu u^3,u^5),\qquad
\beta_{\rm cl}=-\eta\chi-\alpha_3+\frac{2(1+\eta)}{\eta}\alpha_2^2.
\label{eq:cluster_amplitude_full}
\end{equation}
If $\mu>0$ and $\beta_{\rm cl}>0$, then $u_\star^2=\mu/\beta_{\rm cl}+O(\mu^2)$ and the branch is stable along the cluster line. Intra-cluster perturbations have multipliers $1+\mu\pm2\alpha_2u_\star+O(\mu)$, so the balanced state is stable near threshold only if $\alpha_2=0$; in that case it is stable if $\alpha_3<0$ (and marginal at cubic order when $\alpha_3=0$).
\end{theorem}

Without PE,
\begin{equation}
\Att_0=\frac1N\mathbf1\mathbf1^\top,
\qquad
L=I+\frac\eta N\mathbf1\mathbf1^\top.
\end{equation}
Thus $Lx=x$ on $\mathbf1^\perp$ and $L\mathbf1=L_0\mathbf1$ with $L_0=1+\eta$. The heterogeneous multiplier is $J(q\neq0)=\alpha_1=1+\mu$, whereas the homogeneous multiplier $J(0)=\alpha_1L_0$ remains strictly inside the unit circle. The critical sector is degenerate, so we reduce on the balanced two-cluster direction and then check perturbations transverse to it.

\begin{proof}
\emph{Generated sector and slaving.}
Let $s_i\in\{\pm1\}$ with $\sum_i s_i=0$. Since $s^{\circ2}=\mathbf1$, the quadratic FFN term generated by the critical direction $us$ has no heterogeneous component: it drives only the homogeneous sector. The minimal closure is therefore
\begin{equation}
x=us+m\mathbf1,\qquad m=O(u^2).
\label{eq:cluster_minimal_ansatz}
\end{equation}
For $x=us$, balanced clusters give the exact attention weights
\begin{equation}
A_{ij}(us)=\frac1N\bigl(1+\tanh(\chi u^2)s_is_j\bigr),
\end{equation}
so $A(us)(us)=\chi u^3s+O(u^5)$. Since $m=O(u^2)$, its effect on the attention weights enters only beyond cubic order.

Before separating the two sectors, write the whole cubic update on the
minimal subspace. Because
\begin{equation}
L(us+m\mathbf1)=us+L_0m\mathbf1,
\end{equation}
we have, through cubic order,
\begin{align}
(Lx)^{\circ2}
&=u^2\mathbf1+2L_0mu\,s+O(u^4),\\
(Lx)^{\circ3}
&=u^3s+O(u^4),\\
\mathcal C(x)
&=u^3s+O(u^4).
\end{align}
Hence Lemma~\ref{lem:scalar_cubic_expansion} gives the unprojected expression
\begin{align}
F(us+m\mathbf1)
&=\Big[J(0)m+\alpha_2u^2\Big]\mathbf1\nonumber\\
&\quad+\Big[\alpha_1u+2\alpha_2L_0mu
 +(\alpha_3+\alpha_1\eta\chi)u^3\Big]s
 +O(\mu u^3,u^4).
\label{eq:cluster_unprojected_cubic}
\end{align}
The two displayed directions are linearly distinct: the first is the stable
homogeneous sector and the second lies in the critical heterogeneous sector.
Reading the coefficient of $\mathbf1$ therefore gives
\begin{equation}
m^+=J(0)m+\alpha_2u^2+O(\mu u^2,u^4).
\end{equation}
Because $|J(0)|<1$ while $u$ changes only by $O(\mu u,u^3)$, the mean follows the slowly varying cluster amplitude:
\begin{equation}
m=\frac{\alpha_2}{1-J(0)}u^2+O(u^4)
  =-\frac{\alpha_2}{\eta}u^2+O(\mu u^2,u^4),
\label{eq:cluster_slaved_mean}
\end{equation}
where the second equality evaluates the stable resolvent at onset.

\smallskip\noindent\emph{Projection onto the cluster amplitude.}
We now read the coefficient of $s$ in Eq.~\eqref{eq:cluster_unprojected_cubic}.
The three cubic contributions are visible before the projection: content
attention, the local cubic FFN, and the quadratic FFN evaluated on the slaved
mean. Using $\alpha_1=1+\mu$ and absorbing the $O(\mu u^3)$ correction gives
\begin{align}
u^+-u
&=\mu u+
\left[
\underbrace{\eta\chi}_{\text{content attention}}
+\underbrace{\alpha_3}_{\text{cubic FFN}}
+\underbrace{2\alpha_2(1+\eta)\frac{m}{u^2}}_{\text{slaved mean}}
\right]u^3
+O(\mu u^3,u^5)\nonumber\\
&=\mu u+
\left(\eta\chi+\alpha_3-\frac{2(1+\eta)}{\eta}\alpha_2^2\right)u^3
+O(\mu u^3,u^5)=\mu u-\beta_{\rm cl}u^3+O(\mu u^3,u^5),
\end{align}
which is Eq.~\ref{eq:cluster_amplitude_full}. For $\mu>0$ and $\beta_{\rm cl}>0$,
$u_\star^2=\mu/\beta_{\rm cl}+O(\mu^2)$, and the multiplier along the cluster line is $1-2\mu+O(\mu^2)$.

\smallskip\noindent\emph{Transverse stability.}
Because the entire zero-mean space is critical without PE, retain all of $\mathbf1^\perp$. For a perturbation $\delta\in\{\mathbf1,s\}^\perp$ at $x_\star=u_\star s+m_\star\mathbf1$, Eq.~\eqref{eq:app_C_def} gives $D\mathcal C(x_\star)\delta=u_\star^2\delta$, and $L\delta=\delta$. Hence
\begin{equation}
DF(x_\star)\delta
=\alpha_1\delta+2\alpha_2u_\star\,s\circ\delta
 +2\alpha_2L_0m_\star\delta
 +3\alpha_3u_\star^2\delta+\alpha_1\eta\chi u_\star^2\delta
 +O(u_\star^3)\delta.
\end{equation}
The involution $\delta\mapsto s\circ\delta$ has eigenvalues $\pm1$ on this transverse space, so
\begin{equation}
\Lambda_\perp^{\pm}=1+\mu\pm2\alpha_2u_\star
 +(2\alpha_3-\beta_{\rm cl})u_\star^2+O(u_\star^3).
\end{equation}
If $\alpha_2\neq0$, one multiplier exceeds one because $u_\star=O(\mu^{1/2})\gg\mu$. If $\alpha_2=0$,
\begin{equation}
\Lambda_\perp
=1+\frac{2\alpha_3}{\beta_{\rm cl}}\mu+O(\mu^2),
\end{equation}
which is below one for $\beta_{\rm cl}>0$ exactly when $\alpha_3<0$; $\alpha_3=0$ is marginal at cubic order. This proves the theorem.
\end{proof}

\section{Standing waves}
\label{app:standing}
The standing-wave case is the cleanest illustration of the general mechanism: one critical Fourier pair quadratically generates a mean and a second harmonic, both of which are linearly stable and therefore slaved.

\begin{proposition}[Standing waves]
\label{prop:standing_amplitude}
Assume a reflection-symmetric scalar positional kernel and a simple non-resonant critical pair $\pm q_c$ with
\begin{equation}
J(q_c)=1+\mu,
\qquad
|J(q)|<1\quad(q\notin\{\pm q_c\})
\end{equation}
near $\mu=0$. Then the critical Fourier coefficient obeys
\begin{equation}
B^+-B=\mu B-\beta_{\rm st}|B|^2B+O(\mu|B|^3,|B|^5),
\label{eq:standing_landau}
\end{equation}
where, with $\lambda_c:=\lambda_{q_c}$, $\lambda_2:=\lambda_{2q_c}$ and $L_c:=L_{q_c}$,
\begin{align}
-\beta_{\rm st}
&=\alpha_1\eta\chi(2+\lambda_2-3\lambda_c^2)+3\alpha_3L_c^3\nonumber\\
&\quad+2\alpha_2^2L_c^3
\left[\frac{2L_0}{1-J(0)}+\frac{L_{2q_c}}{1-J(2q_c)}\right],
\label{eq:standing_beta_general}
\end{align}
all quantities on the right being evaluated at $\mu=0$. If $\mu>0$ and $\beta_{\rm st}>0$, then $|B_\star|^2=\mu/\beta_{\rm st}+O(\mu^2)$; the phase is stationary.
\end{proposition}

\begin{proof}
Reflection symmetry makes all $\lambda_q$ and $J(q)$ real. Let $L_2:=L_{2q_c}$.

\smallskip\noindent\emph{Generated sectors and slaving.}
The quadratic product of the critical pair generates only the mean and second harmonic. We therefore truncate the invariant graph at quadratic order,
\begin{equation}
x^{[2]}_j=B\euler^{\im q_cj}+\overline B\euler^{-\im q_cj}
    +m+Z\euler^{2\im q_cj}+\overline Z\euler^{-2\im q_cj}.
\end{equation}
The omitted cubic stable correction contains, in particular, the generated
third harmonic, but it cannot feed back into $q_c$ at cubic order. Here all $L_q$ and $\lambda_q$ are real. Substituting this truncation into
Lemma~\ref{lem:scalar_cubic_expansion} and collecting equal Fourier sectors gives
the complete cubic update before any sector is discarded:
\begin{align}
F(x^{[2]})
&=\Big[J(0)m+2\alpha_2L_c^2|B|^2\Big]e_0\nonumber\\
&\quad+\Big\{\Big[J(q_c)B
 +2\alpha_2L_c\big(L_0mB+L_2Z\overline B\big)
 +3\alpha_3L_c^3|B|^2B
 +\alpha_1\eta\chi(2+\lambda_2-3\lambda_c^2)|B|^2B\Big]e_{q_c}
 +\mathrm{c.c.}\Big\}\nonumber\\
&\quad+\Big\{\big[J(2q_c)Z+\alpha_2L_c^2B^2\big]e_{2q_c}
 +\mathrm{c.c.}\Big\}\nonumber\\
&\quad+\Big\{\Big[2\alpha_2L_cL_2BZ
 +\big(\alpha_3L_c^3+\alpha_1\eta\chi(\lambda_2-\lambda_c^2)\big)B^3\Big]e_{3q_c}
 +\mathrm{c.c.}\Big\}
 +O(|B|^4).
\label{eq:standing_unprojected_cubic}
\end{align}
This makes the bookkeeping visible. The $0$ and $2q_c$ sectors are forced at
quadratic order and are stable; $q_c$ receives the cubic feedback that controls
saturation; and $3q_c$ is generated at cubic order but cannot return to $q_c$
at the same order.

Reading the $0$ and $2q_c$ lines of Eq.~\eqref{eq:standing_unprojected_cubic},
\begin{equation}
m^+=J(0)m+2\alpha_2L_c^2|B|^2+O(|B|^4),
\qquad
Z^+=J(2q_c)Z+\alpha_2L_c^2B^2+O(|B|^4).
\end{equation}
At stationary onset $J(q_c)=1$, so
\begin{equation}
m=\frac{2\alpha_2L_c^2}{1-J(0)}|B|^2+O(|B|^4),
\qquad
Z=\frac{\alpha_2L_c^2}{1-J(2q_c)}B^2+O(|B|^4).
\label{eq:standing_slaved_modes}
\end{equation}
The two denominators are exactly the stable linear resolvents of the generated sectors.

\smallskip\noindent\emph{Projection onto the critical mode.}
We now read the $q_c$ line of Eq.~\eqref{eq:standing_unprojected_cubic}. In
projection notation, its three nonlinear pieces are
\begin{align}
\Pi_{q_c}\mathcal C(x_c)
&=(2+\lambda_2-3\lambda_c^2)|B|^2B,\\
\Pi_{q_c}(Lx_c)^{\circ3}
&=3L_c^3|B|^2B,\\
\Pi_{q_c}\!\bigl[(Lx)^{\circ2}-(Lx_c)^{\circ2}\bigr]
&=2L_c\bigl(L_0mB+L_2Z\overline B\bigr)+O(|B|^4).
\end{align}
Thus the projections are simply a compact way of reading terms that are
already explicit in Eq.~\eqref{eq:standing_unprojected_cubic}. Substituting
Eq.~\eqref{eq:standing_slaved_modes} into the last line and weighting the three
pieces by the coefficients in Lemma~\ref{lem:scalar_cubic_expansion} gives
\begin{align}
-\beta_{\rm st}
&=\underbrace{\alpha_1\eta\chi(2+\lambda_2-3\lambda_c^2)}_{\text{content attention}}
 +\underbrace{3\alpha_3L_c^3}_{\text{cubic FFN}}\nonumber\\
&\quad+
\underbrace{2\alpha_2^2L_c^3
\left[\frac{2L_0}{1-J(0)}+\frac{L_2}{1-J(2q_c)}\right]}_{\text{slaved mean and second harmonic}},
\end{align}
which is Eq.~\eqref{eq:standing_beta_general}. Hence Eq.~\eqref{eq:standing_landau} follows. All coefficients are real, so the phase is stationary. On the nonzero branch the radial multiplier is $1-2\mu+O(\mu^2)$, while the slaved sectors retain multipliers $J(0)+O(\mu)$ and $J(2q_c)+O(\mu)$ strictly inside the unit circle.
\end{proof}

\section{Traveling waves: proof of Theorem~\ref{prop:travelling_amplitude}}
\label{app:travelling}
We first state the full version of Theorem~\ref{prop:travelling_amplitude}.
\begin{theorem}[Traveling waves, full version of Theorem~\ref{prop:travelling_amplitude}]
\label{thm:travelling_full}
Assume a simple non-resonant critical pair $\pm q_c$ for a directional kernel. Write its scalar dispersion gain as
\begin{equation}
    J(q_c)=(1+\mu)\euler^{\im\omega_c},
\end{equation}
with $\mu=0$ at onset and $\omega_c\neq0$, and no strong resonance: $\euler^{\im k\omega_c}\neq1$ for $k=1,\dots,4$. If $B_\ell$ is the critical Fourier coefficient and $B_\ell=\euler^{\im\ell\omega_c}G_\ell$, then
\begin{equation}
G^+-G=\mu G-\beta_{\rm tr}|G|^2G
+O(\mu|G|^3,|G|^5),
\label{eq:travelling_landau}
\end{equation}
where the explicit coefficient $\beta_{\rm tr}=-\euler^{-\im\omega_c}\beta^{\rm tr}_{q_c}$ is derived below, with $\beta^{\rm tr}_{q_c}$ in Eq.~\ref{eq:travelling_beta_general}. The second harmonic is bound in a frame rotating at $2\omega_c$. If $\mu>0$ and $\Re(\beta_{\rm tr})>0$, the amplitude saturates at
\begin{equation}
R_\star^2=\frac{\mu}{\Re(\beta_{\rm tr})}+O(\mu^2),\qquad
\Omega=\omega_c-\Im(\beta_{\rm tr})R_\star^2+O(R_\star^4).
\label{eq:nonlinear_wave_speed_full}
\end{equation}
Thus the real part of $\beta_{\rm tr}$ controls saturation, while its imaginary part gives the nonlinear correction to the propagation speed.
\end{theorem}

For a directional kernel, $\lambda_{-q}=\overline{\lambda_q}$ and $L_{-q}=\overline{L_q}$. At onset let
\begin{equation}
J_c:=J(q_c)=\euler^{\im\omega_c},
\qquad
\lambda_c:=\lambda_{q_c},\quad
\lambda_2:=\lambda_{2q_c},\quad
L_c:=L_{q_c},\quad L_2:=L_{2q_c}.
\end{equation}
The generated sectors are the same as for a standing wave, but the second harmonic now rotates twice as fast as the critical mode.

\begin{proof}
The reduction is identical to the standing-wave calculation except for the linear phase accumulated by the generated harmonics.

\smallskip\noindent\emph{Generated sectors and slaving.}
Use the quadratic truncation
\begin{equation}
x^{[2]}_j=B\euler^{\im q_cj}+\overline B\euler^{-\im q_cj}
    +m+Z\euler^{2\im q_cj}+\overline Z\euler^{-2\im q_cj}.
\end{equation}
As above, the omitted cubic stable correction cannot feed back into $q_c$ at
cubic order. Now $L_{-q_c}=\overline L_c$ and
$L_{-2q_c}=\overline L_2$. The full cubic update of $x^{[2]}$, before
projecting onto any Fourier sector, is
\begin{align}
F(x^{[2]})
&=\Big[J(0)m+2\alpha_2|L_c|^2|B|^2\Big]e_0\nonumber\\
&\quad+\Big\{\Big[J_cB
 +2\alpha_2\big(L_cL_0mB+\overline L_cL_2Z\overline B\big)
 +3\alpha_3|L_c|^2L_c|B|^2B\nonumber\\
&\hspace{7.6em}
 +\alpha_1\eta\chi\big(2(1-|\lambda_c|^2)+\lambda_2-\lambda_c^2\big)|B|^2B\Big]e_{q_c}
 +\mathrm{c.c.}\Big\}\nonumber\\
&\quad+\Big\{\big[J(2q_c)Z+\alpha_2L_c^2B^2\big]e_{2q_c}
 +\mathrm{c.c.}\Big\}\nonumber\\
&\quad+\Big\{\Big[2\alpha_2L_cL_2BZ
 +\big(\alpha_3L_c^3+\alpha_1\eta\chi(\lambda_2-\lambda_c^2)\big)B^3\Big]e_{3q_c}
 +\mathrm{c.c.}\Big\}
 +O(|B|^4).
\label{eq:travelling_unprojected_cubic}
\end{align}
The generated sectors are therefore exactly the same as in the standing case.
What changes is their layer-to-layer phase: the mean is stationary, while the
$2q_c$ coefficient forced by $B^2$ rotates with $J_c^2=e^{2\im\omega_c}$.

Reading the first and third lines of
Eq.~\eqref{eq:travelling_unprojected_cubic} gives
\begin{equation}
m^+=J(0)m+2\alpha_2|L_c|^2|B|^2+O(|B|^4),
\qquad
Z^+=J(2q_c)Z+\alpha_2L_c^2B^2+O(|B|^4).
\end{equation}
Hence
\begin{equation}
m=\frac{2\alpha_2|L_c|^2}{1-J(0)}|B|^2+O(|B|^4),
\qquad
Z=\frac{\alpha_2L_c^2}{\euler^{2\im\omega_c}-J(2q_c)}B^2+O(|B|^4).
\label{eq:travelling_slaved_modes}
\end{equation}
Relative to Eq.~\eqref{eq:standing_slaved_modes}, only the second-harmonic resolvent changes: it is evaluated in the frame rotating at twice the critical phase.

\smallskip\noindent\emph{Projection onto the critical mode.}
We now read the $q_c$ line of Eq.~\eqref{eq:travelling_unprojected_cubic}.
The same three mechanisms contribute, and in projection notation they are
\begin{align}
\Pi_{q_c}\mathcal C(x_c)
&=\left[2(1-|\lambda_c|^2)+\lambda_2-\lambda_c^2\right]|B|^2B,\\
\Pi_{q_c}(Lx_c)^{\circ3}
&=3|L_c|^2L_c|B|^2B,\\
\Pi_{q_c}\!\bigl[(Lx)^{\circ2}-(Lx_c)^{\circ2}\bigr]
&=2\bigl(L_cL_0mB+\overline L_cL_2Z\overline B\bigr)+O(|B|^4).
\end{align}
The complex conjugate on $L_c$ in the second-harmonic feedback is essential:
it is the $(-q_c)+(2q_c)=q_c$ interaction. Using
Eq.~\eqref{eq:travelling_slaved_modes} therefore gives, at onset,
\begin{align}
\beta_{q_c}^{\rm tr}
&=\underbrace{\alpha_1\eta\chi
   \left[2(1-|\lambda_c|^2)+\lambda_2-\lambda_c^2\right]}_{\text{content attention}}
  +\underbrace{3\alpha_3|L_c|^2L_c}_{\text{cubic FFN}}\nonumber\\
&\quad+
\underbrace{2\alpha_2^2L_c|L_c|^2
\left[\frac{2L_0}{1-J(0)}
      +\frac{L_2}{\euler^{2\im\omega_c}-J(2q_c)}\right]}_{\text{slaved mean and second harmonic}}.
\label{eq:travelling_beta_general}
\end{align}
Thus
\begin{equation}
B^+=J(q_c)B+\beta_{q_c}^{\rm tr}|B|^2B+O(\mu|B|^3,|B|^5).
\end{equation}
Writing $J(q_c)=(1+\mu)\euler^{\im\omega_c}$ and $B_\ell=\euler^{\im\ell\omega_c}G_\ell$ removes the linear rotation:
\begin{equation}
G^+-G=\mu G+\beta_{\rm tr}|G|^2G+O(\mu|G|^3,|G|^5),
\qquad
\beta_{\rm tr}:=\euler^{-\im\omega_c}\beta_{q_c}^{\rm tr}.
\end{equation}
Writing $G=R\euler^{\im\theta}$ then gives
\begin{equation}
R^+-R=\mu R-\Re(\beta_{\rm tr})R^3+O(\mu R^3,R^5),
\qquad
\theta^+-\theta=-\Im(\beta_{\rm tr})R^2+O(R^4),
\end{equation}
from which Eq.~\ref{eq:nonlinear_wave_speed_full} follows. For a joint token--feature mode, projection onto the corresponding left eigenvector gives the same co-rotating normal form; a complex right eigenvector produces the feature-plane rotation described by Eq.~\eqref{eq:vector_wave_form}.
\end{proof}

\section{Fixed points away from the origin}
\label{app:offorigin}

The previous analysis assumes $a_\star=0$. The same reduction applies around a nonzero homogeneous fixed point $a_\star\neq0$, as may arise with biases or LayerNorm. The only qualitative difference is that content attention then contributes already at quadratic order.

Write

$$
x=a_\star\mathbf1+u,
\qquad
y_\star=(1+\eta)a_\star,
\qquad
A_0:=A(a_\star\mathbf1),
\qquad
L:=I+\eta A_0,
$$

with

$$
\alpha_k:=\delta_{k1}+\frac{r}{k!}\phi^{(k)}(y_\star).
$$

For each attention row define

$$
\kappa_{2,i}(u):=(A_0u^{\circ2})_i-(A_0u)_i^2,
$$

and

$$
\kappa_{3,i}(u):=(A_0u^{\circ3})_i
-3(A_0u^{\circ2})_i(A_0u)_i
+2(A_0u)_i^3.
$$

\paragraph{Proof sketch.}
For a fixed row $i$,

$$
\chi x_ix_j
=
\underbrace{\chi a_\star(a_\star+u_i)}_{\text{independent of }j}
+\chi(a_\star+u_i)u_j,
$$

so the first term cancels from the row softmax. The perturbed attention row is therefore an exponential tilt of $A_0$. Expanding its mean in cumulants gives

$$
[A(x)x]_i
=
a_\star+(A_0u)_i
+\chi a_\star\kappa_{2,i}(u)
+\chi u_i\kappa_{2,i}(u)
+\frac12(\chi a_\star)^2\kappa_{3,i}(u)
+O(\|u\|^4).
$$

Expanding the FFN around $y_\star$ then yields

$$
u^+
=
\alpha_1Lu+\mathcal N_2(u)+\mathcal N_3(u)+O(\|u\|^4),
$$

with

$$
\mathcal N_2(u)
=
\alpha_2(Lu)^{\circ2}
+\alpha_1\eta\chi a_\star\kappa_2(u),
$$

and

$$
\mathcal N_3(u)
=
\alpha_3(Lu)^{\circ3}
+\alpha_1\eta\chi\,u\circ\kappa_2(u)
+\frac12\alpha_1\eta(\chi a_\star)^2\kappa_3(u)
+2\alpha_2\eta\chi a_\star(Lu)\circ\kappa_2(u).
$$

Hence the center-manifold reduction is unchanged: the linear dispersion relation still determines the critical modes, while nonlinear interactions generate stable harmonics that slave to them. Compared with $a_\star=0$, the new term

$$
\alpha_1\eta\chi a_\star\kappa_2(u)
$$

adds an attention-mediated quadratic forcing of the slaved modes. For example, for a standing critical pair $\pm q_c$, the mean and second-harmonic drives become

$$
p_0=\alpha_2L_c^2+\alpha_1\eta\chi a_\star(1-\lambda_c^2),
\qquad
p_2=\alpha_2L_c^2+\alpha_1\eta\chi a_\star(\lambda_2-\lambda_c^2).
$$

The same slaving and projection steps therefore give

$$
B^+-B=\mu B+\beta_{q_c}^{\rm st}|B|^2B+\cdots
$$

with modified coefficients. Setting $a_\star=0$ recovers the previous analysis.

\section{Proof of Theorem~\ref{prop:mode_competition}}
\label{app:competition}
We first state the full version of Theorem~\ref{prop:mode_competition}.
\begin{theorem}[Two-mode competition, full version of Theorem~\ref{prop:mode_competition}]
\label{thm:competition_full}
In the reflection-symmetric scalar reduction, let two distinct standing modes $q_1,q_2$ be near critical, with all other sectors generated up to cubic order strictly stable and with no resonances $n_1q_1+n_2q_2\equiv0\pmod N$ for $0<|n_1|+|n_2|\le4$. Their amplitudes obey
\begin{align}
B_1^+-B_1&=(\mu_1+\beta_{11}|B_1|^2+\beta_{12}|B_2|^2)B_1+\cdots,\nonumber\\
B_2^+-B_2&=(\mu_2+\beta_{22}|B_2|^2+\beta_{21}|B_1|^2)B_2+\cdots,
\label{eq:two_mode_amplitudes}
\end{align}
where $\mu_i=J(q_i)-1$, and the self-coupling $\beta_{ii}$ and cross-coupling $\beta_{ij}$, negative when saturating and suppressing respectively, are computed below. At simultaneous onset the cross-couplings are symmetric to cubic order, $\beta_{12}=\beta_{21}$, so the moduli $R_i=|B_i|$ obey
\begin{equation}
R_i^+-R_i=-\frac{\partial\mathcal V}{\partial R_i},
\qquad
\mathcal V(R_1,R_2)=-\sum_{i=1,2}\frac{\mu_i}{2}R_i^2-\sum_{i=1,2}\frac{\beta_{ii}}{4}R_i^4-\frac{\beta_{12}}{2}R_1^2R_2^2 .
\label{eq:two_mode_potential_full}
\end{equation}
In the symmetric case $\mu_i=\mu>0$, $\beta_{ii}=-\beta_{\rm self}$ and $\beta_{ij}=-\beta_{\rm cross}$ with $\beta_{\rm self}>0$: if $\beta_{\rm cross}<\beta_{\rm self}$, the mixed state $R_1^2=R_2^2=\mu/(\beta_{\rm self}+\beta_{\rm cross})$ is stable and the two modes coexist; if $\beta_{\rm cross}>\beta_{\rm self}$, the two single-mode states $R_i^2=\mu/\beta_{\rm self}$ are stable and the initial condition selects the winner.
\end{theorem}

Assume a reflection-symmetric scalar kernel and two distinct standing modes $q_1,q_2$ simultaneously near onset, all other sectors being strictly stable. Write
\begin{equation}
x_j=B_1\euler^{\im q_1j}+\overline B_1\euler^{-\im q_1j}
    +B_2\euler^{\im q_2j}+\overline B_2\euler^{-\im q_2j}
    +O(|B|^2).
\end{equation}
The non-resonance condition of Theorem~\ref{thm:competition_full} ensures that the quadratic products generate only stable sectors
\begin{equation}
0,\qquad \pm2q_1,\qquad \pm2q_2,\qquad \pm(q_1+q_2),\qquad \pm(q_1-q_2),
\end{equation}
and that no other cubic wave-vector identity folds into $\pm q_1$ or $\pm q_2$.

\begin{proof}
\emph{Generated sectors and slaving.}
Define the critical part
\begin{equation}
x_c:=B_1e_{q_1}+\overline B_1e_{-q_1}
    +B_2e_{q_2}+\overline B_2e_{-q_2}.
\end{equation}
Let $L_i:=L_{q_i}$ and write the $O(|B|^2)$ stable correction explicitly as
\begin{align}
x_s
&=m+Z_1e_{2q_1}+\overline Z_1e_{-2q_1}
     +Z_2e_{2q_2}+\overline Z_2e_{-2q_2}\nonumber\\
&\quad+W_+e_{q_1+q_2}+\overline W_+e_{-(q_1+q_2)}
     +W_-e_{q_1-q_2}+\overline W_-e_{-(q_1-q_2)}.
\label{eq:competition_stable_ansatz}
\end{align}
Before projecting anything, the quadratic product of the two critical pairs is
\begin{align}
(Lx_c)^{\circ2}
&=2\big(L_1^2|B_1|^2+L_2^2|B_2|^2\big)e_0\nonumber\\
&\quad+\Big[L_1^2B_1^2e_{2q_1}+L_2^2B_2^2e_{2q_2}
 +2L_1L_2B_1B_2e_{q_1+q_2}
 +2L_1L_2B_1\overline B_2e_{q_1-q_2}
 +\mathrm{c.c.}\Big].
\label{eq:competition_quadratic_expansion}
\end{align}
Thus the list of stable sectors in the theorem is read directly from the
unprojected quadratic term. Through cubic order the full map is
\begin{align}
F(x_c+x_s)
&=\alpha_1L(x_c+x_s)
 +\alpha_2(Lx_c)^{\circ2}
 +2\alpha_2(Lx_c)\circ(Lx_s)\nonumber\\
&\quad+\alpha_3(Lx_c)^{\circ3}
 +\alpha_1\eta\chi\,\mathcal C(x_c)
 +O(|B|^4).
\label{eq:competition_unprojected_cubic}
\end{align}
The two direct cubic terms are supported on
\begin{equation}
\pm q_1,\ \pm q_2,\ \pm3q_1,\ \pm3q_2,
\ \pm(2q_1\pm q_2),\ \pm(q_1\pm2q_2).
\end{equation}
The non-resonance assumption ensures that none of the noncritical entries in
this list folds back onto $\pm q_1$ or $\pm q_2$. Therefore the quadratically
generated modes in Eq.~\eqref{eq:competition_stable_ansatz} are the only
stable modes that can feed back into a critical equation at cubic order.

At simultaneous standing onset $J(q_i)=1$. Reading the coefficients of the
quadratically generated sectors in
Eq.~\eqref{eq:competition_unprojected_cubic}  gives
\begin{align}
m&=\frac{2\alpha_2\left(L_1^2|B_1|^2+L_2^2|B_2|^2\right)}{1-J(0)},\\
Z_i&=\frac{\alpha_2L_i^2}{1-J(2q_i)}B_i^2,\\
W_+&=\frac{2\alpha_2L_1L_2}{1-J(q_1+q_2)}B_1B_2,
\qquad
W_- =\frac{2\alpha_2L_1L_2}{1-J(q_1-q_2)}B_1\overline B_2.
\label{eq:competition_slaved_modes}
\end{align}
The self-interaction of mode $q_i$ is exactly the standing-wave calculation of Proposition~\ref{prop:standing_amplitude}; its cubic coefficient is $\beta_{ii}=-\beta_{\rm st}(q_i)$,  so that $\beta_{\rm self}=\beta_{\rm st}$.

\smallskip\noindent\emph{Projection onto one critical mode.}
Fix $i\neq j$. To make the projection transparent, first display the terms of
each unprojected polynomial in Eq.~\eqref{eq:competition_unprojected_cubic}
that lie in sector $q_i$:
\begin{align}
2(Lx_c)\circ(Lx_s)
&\supset2\Big[
 L_iL_0mB_i+L_iL_{2q_i}Z_i\overline B_i
 +L_jL_{q_i+q_j}W_+\overline B_j
 +L_jL_{q_i-q_j}W_-B_j\Big]e_{q_i},
\label{eq:competition_quadratic_feedback_raw}\\
(Lx_c)^{\circ3}
&\supset\Big[3L_i^3|B_i|^2B_i+6L_iL_j^2|B_j|^2B_i\Big]e_{q_i},
\label{eq:competition_cubic_ffn_raw}\\
\mathcal C(x_c)
&\supset\Big[
 (2+\lambda_{2q_i}-3\lambda_{q_i}^2)|B_i|^2B_i\nonumber\\
&\hspace{4.8em}
 +2\big(1-\lambda_{q_j}^2+\lambda_{q_i+q_j}+\lambda_{q_i-q_j}
 -2\lambda_{q_i}\lambda_{q_j}\big)|B_j|^2B_i
 \Big]e_{q_i}.
\label{eq:competition_attention_raw}
\end{align}
All omitted terms lie in the noncritical cubic sectors listed above. Thus
$\Pi_{q_i}$ simply extracts the displayed coefficients. In particular, the
direct cross pieces are
\begin{align}
\Pi_{q_i}(Lx_c)^{\circ3}
&\supset6L_iL_j^2|B_j|^2B_i,\\
\Pi_{q_i}\mathcal C(x_c)
&\supset2\left[1-\lambda_{q_j}^2+\lambda_{q_i+q_j}+\lambda_{q_i-q_j}
-2\lambda_{q_i}\lambda_{q_j}\right]|B_j|^2B_i.
\end{align}
The mean and mixed harmonics in Eq.~\eqref{eq:competition_slaved_modes} feed back through the quadratic FFN. Their contribution is
\begin{equation}
\beta^{\rm slave}_{ij}
=4\alpha_2^2L_iL_j^2
\left[
\frac{L_0}{1-J(0)}
+\frac{L_{q_i+q_j}}{1-J(q_i+q_j)}
+\frac{L_{q_i-q_j}}{1-J(q_i-q_j)}
\right].
\end{equation}
Therefore the cross coefficient is the sum of direct attention, direct cubic FFN, and slaved feedback:
\begin{align}
\beta_{ij}
&=\underbrace{2\alpha_1\eta\chi
\left[1-\lambda_{q_j}^2+\lambda_{q_i+q_j}+\lambda_{q_i-q_j}
-2\lambda_{q_i}\lambda_{q_j}\right]}_{\text{content attention}}\nonumber\\
&\quad+\underbrace{6\alpha_3L_j^2L_i}_{\text{cubic FFN}}
+\underbrace{4\alpha_2^2L_iL_j^2
\left[
\frac{L_0}{1-J(0)}
+\frac{L_{q_i+q_j}}{1-J(q_i+q_j)}
+\frac{L_{q_i-q_j}}{1-J(q_i-q_j)}
\right]}_{\text{slaved mean and mixed harmonics}}.
\label{eq:cross_coupling}
\end{align}
This gives Eq.~\ref{eq:two_mode_amplitudes}.

\smallskip\noindent\emph{Coexistence versus exclusion.}
At simultaneous onset $J(q_1)=J(q_2)=1$ implies $L_1=L_2$; because the symbol is even, Eq.~\eqref{eq:cross_coupling} is symmetric under $i\leftrightarrow j$, so $\beta_{12}=\beta_{21}$ to cubic order. The radial dynamics is therefore the gradient map of Eq.~\eqref{eq:two_mode_potential}. In the symmetric case $\mu_i=\mu>0$, $\beta_{ii}=-\beta_{\rm self}$ and $\beta_{ij}=-\beta_{\rm cross}$ with $\beta_{\rm self}>0$, the mixed state has $R_1^2=R_2^2=\mu/(\beta_{\rm self}+\beta_{\rm cross})$. Linearizing the radial increment there gives
\begin{equation}
-2\mu,
\qquad
-2\mu\,\frac{\beta_{\rm self}-\beta_{\rm cross}}{\beta_{\rm self}+\beta_{\rm cross}},
\end{equation}
so coexistence is stable for $\beta_{\rm cross}<\beta_{\rm self}$. A single-mode state, say $R_1^2=\mu/\beta_{\rm self}$, has transverse multiplier
\begin{equation}
1+\mu\left(1-\frac{\beta_{\rm cross}}{\beta_{\rm self}}\right)+O(\mu^2),
\end{equation}
so it is stable against invasion by mode $2$ exactly when $\beta_{\rm cross}>\beta_{\rm self}$. By symmetry the same holds for the other single-mode state.
\end{proof}

\section{Effective potentials}
\label{app:potential}

The potential representation used in the main text follows directly from
the cubic amplitude equations.

\paragraph{One mode.}
For a single critical mode, the amplitude equation in the co-rotating
frame has the form
\begin{equation}
    G^+-G
    =
    \mu G-\Gamma |G|^2G+\cdots .
\end{equation}
Writing $G=R\euler^{\im\theta}$ gives, to cubic order,
\begin{equation}
    R^+-R
    =
    \mu R-\Re(\Gamma)R^3+\cdots ,
\end{equation}
while $\Im(\Gamma)$ produces the nonlinear phase drift. Defining
\begin{equation}
    \beta:=\Re(\Gamma),
\end{equation}
the radial dynamics becomes
\begin{equation}
    R^+-R
    =
    -\frac{d\mathcal V}{dR},
\end{equation}
with
\begin{equation}
    \mathcal V(R)
    =
    -\frac{\mu}{2}R^2
    +\frac{\beta}{4}R^4.
\end{equation}
For standing waves and clusters the cubic coefficient is real. For
traveling waves, the potential describes only the radial dynamics, while
the imaginary part of $\Gamma$ controls the phase velocity.

\paragraph{Two competing modes.}
For two non-resonant standing modes, the amplitude equations are
\begin{align}
    R_1^+-R_1
    &=
    \left(
    \mu_1+\beta_{11}R_1^2+\beta_{12}R_2^2
    \right)R_1+\cdots ,
    \\
    R_2^+-R_2
    &=
    \left(
    \mu_2+\beta_{22}R_2^2+\beta_{21}R_1^2
    \right)R_2+\cdots .
\end{align}
At onset, in the reflection-symmetric setting considered here, the two
critical modes have the same linear multiplier and the cross-coupling is
symmetric,
\begin{equation}
    \beta_{12}=\beta_{21}.
\end{equation}
The radial dynamics can therefore be written as
\begin{equation}
    R_i^+-R_i
    =
    -\frac{\partial\mathcal V}{\partial R_i},
\end{equation}
with
\begin{equation}
    \mathcal V(R_1,R_2)
    =
    -\frac{\mu_1}{2}R_1^2
    -\frac{\mu_2}{2}R_2^2
    -\frac{\beta_{11}}{4}R_1^4
    -\frac{\beta_{22}}{4}R_2^4
    -\frac{\beta_{12}}{2}R_1^2R_2^2.
\end{equation}

Now consider the symmetric saturating case
\begin{equation}
    \mu_1=\mu_2=\mu>0,
    \qquad
    \beta_{11}=\beta_{22}=-\beta_{\rm self},
    \qquad
    \beta_{12}=-\beta_{\rm cross},
\end{equation}
with $\beta_{\rm self}>0$. Then
\begin{equation}
    \mathcal V(R_1,R_2)
    =
    -\frac{\mu}{2}(R_1^2+R_2^2)
    +\frac{\beta_{\rm self}}{4}(R_1^4+R_2^4)
    +\frac{\beta_{\rm cross}}{2}R_1^2R_2^2.
\end{equation}
Its stationary points satisfy
\begin{align}
    R_1
    \left(
    -\mu+\beta_{\rm self}R_1^2
    +\beta_{\rm cross}R_2^2
    \right)
    &=0,
    \\
    R_2
    \left(
    -\mu+\beta_{\rm self}R_2^2
    +\beta_{\rm cross}R_1^2
    \right)
    &=0.
\end{align}

Besides the origin, there are two single-mode states,
\begin{equation}
    (R_1^2,R_2^2)
    =
    \left(
    \frac{\mu}{\beta_{\rm self}},0
    \right),
    \qquad
    \left(
    0,\frac{\mu}{\beta_{\rm self}}
    \right),
\end{equation}
and a mixed state
\begin{equation}
    R_1^2=R_2^2
    =
    \frac{\mu}
    {\beta_{\rm self}+\beta_{\rm cross}},
\end{equation}
provided $\beta_{\rm self}+\beta_{\rm cross}>0$.

At either single-mode state, the Hessian of $\mathcal V$ has eigenvalues
\begin{equation}
    2\mu,
    \qquad
    \mu
    \left(
    \frac{\beta_{\rm cross}}
    {\beta_{\rm self}}
    -1
    \right).
\end{equation}
Therefore the single-mode states are minima when
\begin{equation}
    \beta_{\rm cross}>\beta_{\rm self}.
\end{equation}
In this regime, cross-suppression is stronger than self-saturation, so
coexistence is unstable and either mode can suppress the other: the
winner-takes-all regime.

At the mixed state, the Hessian eigenvalues are
\begin{equation}
    2\mu,
    \qquad
    2\mu
    \frac{
    \beta_{\rm self}-\beta_{\rm cross}
    }{
    \beta_{\rm self}+\beta_{\rm cross}
    }.
\end{equation}
Hence the mixed state is a minimum when
\begin{equation}
    \beta_{\rm cross}<\beta_{\rm self},
\end{equation}
corresponding to stable coexistence of the two modes.

The transition occurs at
\begin{equation}
    \beta_{\rm cross}
    =
    \beta_{\rm self},
\end{equation}
where the cubic potential becomes degenerate along
\begin{equation}
    R_1^2+R_2^2
    =
    \frac{\mu}{\beta_{\rm self}}.
\end{equation}

Finally, near threshold these minima coincide with the stable fixed points
of the cubic amplitude map. Indeed, linearizing
\begin{equation}
    R^+-R
    =
    -\nabla\mathcal V(R)
\end{equation}
around a stationary point $R_\star$ gives
\begin{equation}
    \delta R^+
    =
    \left[
    I-D^2\mathcal V(R_\star)
    \right]\delta R.
\end{equation}
Since the Hessian eigenvalues are $O(\mu)$, for sufficiently small
$\mu>0$, a local minimum of $\mathcal V$ has all multipliers inside the
unit circle. Thus the potential picture directly reproduces the
coexistence and winner-takes-all regimes of the two-mode amplitude
equations.
\section{Experiments}
\subsection{Effective-potential figure (Figure~\ref{fig:potential})}
\label{app:fig_potential}
Figure~\ref{fig:potential} plots the effective potentials introduced in Sec.~\ref{sec:amplitude}. For one critical mode, panels (a) and (b) show
\[
\mathcal V(B)=-\frac{\mu}{2}B^2+\frac{\beta}{4}B^4,
\qquad B^+-B=-\mathcal V'(B),
\]
for $\mu=-1$ and $\mu=1$, respectively, with $\beta=-1$ and $\beta=1$ overlaid. The marked extrema illustrate stable and unstable amplitudes below and above the linear threshold. For two critical modes, panels (c) and (d) show schematic surfaces of
\[
\mathcal V(R_1,R_2)=-\sum_{i=1}^2\frac{\mu_i}{2}R_i^2
 +\sum_{i=1}^2\frac{\beta_{\rm self}}{4}R_i^4
 +\frac{\beta_{\rm cross}}{2}R_1^2R_2^2,
\]
with $\mu_1=\mu_2=\beta_{\rm self}=1$. When $\beta_{\rm cross}/\beta_{\rm self}=0.4$, the wells lie between the coordinate axes and both modes coexist; when the ratio is $2.5$, the wells lie on the axes and the dynamics is winner-takes-all. These schematic surfaces use signed amplitudes $(R_1,R_2)$ to display the symmetry of the normal form.

\subsection{Scalar-model phase diagrams (Figure~\ref{fig:phase_structure})}
\label{app:fig3}
Figure~\ref{fig:phase_structure} is generated by iterating the scalar map in
Eq.~\eqref{eq:scalar_map} on a ring of $N=32$ tokens. Except for the cluster series in panel (a),  the feed-forward branch
is $r\tanh(H)$ with $r=-0.1$, so that
\begin{equation}
    \alpha_1=C_\star=0.9,
    \qquad
    \alpha_2=0,
    \qquad
    \alpha_3=-\frac{r}{3}=\frac{1}{30}>0.
\end{equation}
Thus the cubic FFN term is destabilizing, and saturation is supplied by the
state dependence of attention through $\eta\chi$, with $\chi=1.5$. We use
$\eta<0$ and plot its magnitude $|\eta|$. The cluster series has no PE, $\eta=-0.6$ and the cubic branch $\phi(y)=\mu y+\alpha_2y^2-y^3/3$ with $\alpha_2=0.3$ (and $r=1$), so that $\alpha_1=1+\mu$, $\alpha_2=0.3$ and $\alpha_3=-1/3$: the heterogeneous sector is critical as in Theorem~\ref{prop:cluster_amplitude}, the even part makes the slaved mean $m=O(\mu)$, and it is started from a balanced two-cluster seed.

The positional logits have the form
\begin{equation}
    b_r=sK_\Delta(r),
\end{equation}
where $s$ is the peak positional logit, or PE strength.
Panel (a) compares the numerical amplitudes and the square-root and linear growth in $\mu$ prediction obtained in Sec.\ref{sec:amplitude} for various representations: wave amplitude $R$, cluster size $u$, the coexisting modes $R_{1,2}$, the angular speed of the traveling wave $\Omega - w_c $ and finally cluster mean $m$.
Panels (b) and (c)
use symmetric two-lobe kernels
$K_\Delta(r)=\euler^{-(r-\Delta)^2/2}+\euler^{-(r+\Delta)^2/2}$, with $\Delta=4$ in panel (b), and $\eta=-0.6$ in panel (c). In panel (a), each series sweeps $\eta$ at a fixed kernel so as to set $\mu$: the standing wave uses $2K_4$ (critical mode $k_c$ the minimizer of $\lambda$); the traveling wave uses the one-sided skewed lobe $\euler^{-(r-4)^2/(2\zeta_\pm^2)}$, with $\zeta_-=0.7$ for $r\le4$ and $\zeta_+=2$ for $r>4$, at $s=1$, whose Fourier symbol is complex and therefore supports traveling modes; and the coexisting modes $k=4,10$ use $4[(1-w_\star)K_2+w_\star K_{4.5}]$ with lobe widths $0.7$ and $w_\star$ tuned so that $\lambda_4=\lambda_{10}$.
Panel (d) uses the mixture
$(1-w)K_{\Delta=4}+wK_{\Delta=3}$, where the second lobe has width $0.7$ and
$|\eta|=0.6$.

At each grid point, we iterate the map for $5000$ layers. The initial state is
a small random perturbation of amplitude $10^{-3}$. In panel (d), we use two single-mode seeds
\begin{equation}
    x_j^{(0)}=0.05\cos(q_kj)+\text{noise},\qquad k\in\{4,5\}.
\end{equation}
After iteration, we Fourier analyse the state. Here $R$ is the modulus of the
critical Fourier coefficient, $\Omega$ is its phase advance over the final
layer, and $k_{\rm obs}$ is the dominant observed wavenumber. In panel (d),
the competition outcome is determined from the final amplitudes of modes $4$
and $5$ for the two initial seeds; bistability means that each seed retains
its own mode. Gray regions indicate parameters for which the homogeneous state
is linearly stable.

The overlays compare these measurements with the predictions of the amplitude
theory. The linear threshold is
\begin{equation}
    \max_q|\Lambda(q)|=1,
    \qquad
    \Lambda(q)=\alpha_1\bigl(1+\eta\lambda(q)\bigr),
\end{equation}
where $\lambda(q)$ is the Fourier symbol of the positional attention kernel.
For a selected standing mode, the predicted saturated amplitude is
\begin{equation}
    R_\star^2
    =\frac{|\Lambda(q_c)|-1}{\beta^{\rm st}_{q_c}}.
\end{equation}
For the traveling-wave series of panel (a), the predicted phase advance and wavenumber is
\begin{equation}
    \Omega
    =\omega_c-\Im(\beta_{\rm tr})R_\star^2,
    \qquad
    k_c=\frac{N}{2\pi}\arg\max_q|\Lambda(q)|.
\end{equation}
For panel (d), the cubic two-mode theory predicts the single-mode states
\begin{equation}
    R_i^2=-\frac{\mu_i}{\beta_{ii}},
    \qquad
    \mu_j<\frac{\beta_{ji}\mu_i}{\beta_{ii}}
\end{equation}
as the condition for the $i$-mode state to suppress mode $j$.

Near onset, where $|\Lambda(q_c)|<1.05$, the simulated amplitude agrees with
$R_\star$ to within $3\%$ in median error; over the full unstable region, the
median error increases to $14\%$, as expected from a cubic truncation. The observed wavenumber matches the linear prediction at $96\%$ of unstable grid points. The exceptions occur near transition lines and in the resonant sector $k_c = 8$, where $4q_c \equiv 0 \pmod{2\pi}$. Finally, the
predicted competition phase agrees with the simulated outcome at $97\%$ of the
grid points.

\section{Controlled experiments: exploiting structured representations}
\label{app:structured}
In this section we give the full construction and evaluation protocol of the controlled experiments described in Section \ref{sec:controlled_experiments}. The objective is to test consequences of the derived matrix dispersion relation as well as our realized taxonomy. The question we seek to answer is the following: if a task requires a privileged representation, does initializing to amplify the relevant joint token-feature structure provide an optimization advantage? To answer this question, we design three task families, and vary a theoretically identifiable property of the critical mode while holding the architecture, parameter count and maximal linear gain fixed. 

\subsection{Architecture and gain-matched dispersion engineering}

All controlled experiments use the following setup. Periodic sequences of $N=32$ tokens, feature width $d=32$, $H=4$ attention heads and $L=8$ looped interactions of the single weight-tied Transformer block detailed in Section \ref{sec:model}. We use relative positional encodings, the standard content-dependent query-key attention as presented, a residual MLP with $\tanh$ activation, and no normalization. We initialize the feedforward weights so that the point-wise residual Jacobian of the homogeneous state satisfies:
 \begin{equation}
     C_*=\xi I
 \end{equation}
with $\xi = 0.9$. For each engineered "morphology" or structured-representation we calibrate $OV$ matrices and relative positional encodings so that:

\begin{equation}
     \max_q \rho(J(q))=\rho_{target}=1.10
\end{equation}
and 
\begin{equation}
    \rho(J(0))=0.95\rho_{target}
\end{equation}
This ensures both that all selected morphologies have the same gain, and also that the homogeneous mode is not preferred at initialization. 

By engineering this dispersion relation, we can controllably select among three morphologies:
\begin{itemize}
    \item  \textbf{Standing-wave}: we place the critical mode at $q_c \in \{ 3,4,5,7\}$
    \item \textbf{Traveling-wave}: in this case all morphologies share the same $q_c=4$ but the leading multiplier is assigned an additional token phase $\theta \in \{ -\frac{\pi}{2}, -\frac{\pi}{4},\frac{\pi}{4},+\frac{\pi}{2} \}$
    \item \textbf{Feature-rotating}: here we also initialize with $q_c=4$ but engineer feature operator for a critical mode inducing a feature-rotation $\psi \in \{ -\frac{\pi}{2}, -\frac{\pi}{4},\frac{\pi}{4},+\frac{\pi}{2} \}$
\end{itemize}
In addition to these engineered initializations we include the following baselines: a gain-matched random morphology with a
seed-dependent random critical mode, a standard Transformer initialization with learned relative positional bias, and
a standard initialization with sinusoidal absolute positional encoding.

\subsection{Task Families}

Here we describe the 3 task families designed to target a specified morphology. 

\paragraph{T1: periodic infilling (wavelength)}: this task consists in infilling a scalar periodic signal whose dominant mode is $q_c \in \{ 3,4,5,7\}$. The signal is perturbed by random phases and amplitude, together with weak nuisance $q=11$ and $q=13$ components. The signal is quantized into 8 levels, giving 8 class labels. In addition to the signal perturbation, a block of 12 tokens is removed, and an additional 25 \% of the token are randomly masked. The masked block is longer than one period for all carrier frequencies, which prevents the task from being solved by mere copying. 

\paragraph{T2: shift (propagation direction)}: input contains a dominant carrier frequency $q_0=4$ and weak nuisance modes $q=3,5$. These signals are again quantized into 8 classes. The output is the input sequence shifted by $s \in \{-2,-1,+1,+2 \}$, 25\% of the tokens are masked. A solution to the task would be an accumulated phase shift of 
\begin{equation}
    \theta_{task} = -\frac{2 \pi q_0s}{N}
\end{equation}

\paragraph{T3: rotating-wave infilling (feature rotation)}:  at token $j$, the input lies in a 2D feature-plane with angle: $\frac{2 \pi q_0 j}{N}+\phi_0$, with $q_0=4$, and is again quantized into 8 angular states giving 8 labels. We add phase perturbation, mask a block of 12 tokens and randomly mask 25 \%  of the tokens. The target simply shifts the feature vectors by $k$ 8th turn with $k \in \{-2, -1, +1, +2 \}$ which would correspond to a rotation in feature space of:

\begin{equation}
    \psi_{task}=\frac{2\pi k}{8}
\end{equation}

The model must then recreate the carrier across token space while ensuring the correct rotation in feature space. 

\subsection{Training protocol and Metrics}

Optimization uses AdamW, with learning rate $2  \times10^{-3}$, weight decay $10^{-4}$, gradient clipping at norm $1$ and batch size of $64$. Each run lasts $300$ gradient steps. For the training-efficiency analysis we employ training data of size:

 \begin{equation}
     n_{train} \in \{ 16,32,64,128,256,512\}
 \end{equation}

\subsection{Results}
\begin{figure}[t]
\centering
\includegraphics[width=\textwidth]{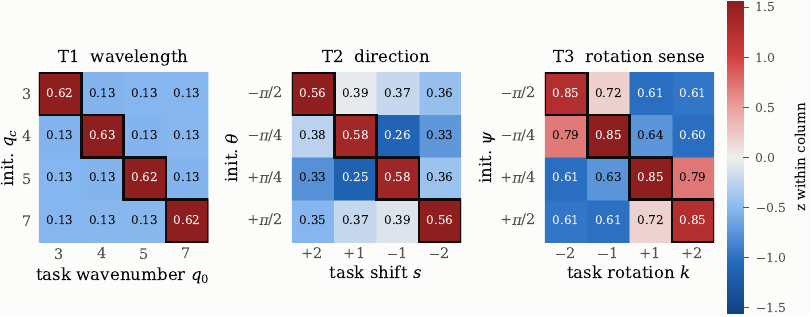}
\caption{Extended version of Fig.~\ref{fig:inductive_prior_results}. Heatmap of AUC of learning curves  for each (task, initialization) pair; color gives the $z$-scores. Higher scores in the diagonal support the dynamical prior induced by mode criticality.}\label{fig:dissociation}
\end{figure}

Figure \ref{fig:dissociation} reports the complete engineered-arm/task grids. The predicted diagonal is visible on all three axes. Wavelength
matching produces an especially sharp dissociation: the matched standing wave arms achieve mean early AUC between
0.619 and 0.626, whereas the frozen strongest mismatched arms are between 0.132 and 0.134. Propagation direction
produces a smaller but consistent phase specificity: matched arms obtain early AUC 0.559–0.582, compared with
0.362–0.391 for the frozen mismatches. Feature-rotation sense is also resolved: matched arms obtain early AUC
approximately 0.853–0.854, with matched-mismatched differences between 0.061 and 0.132 depending on
rotation magnitude.
Across all 12 tasks, every paired matched–mismatched comparison has the same-sign ordering for all 12 final seeds.

\subsection{Learning speed and data efficiency}

Figure \ref{fig:efficiency} gives complementary views of the effect. For T1, the matched morphology improves held-out cross-entropy relative to the gain-matched random and mismatched structured controls across the training-set-size sweep, with the largest differences in the low-data regime. The corresponding learning curves show that most of the separation occurs early, motivating $AUC_{early}$ as the primary statistic. The standard relative-PE baseline is an intentionally less controlled practitioner baseline, it is not gain matched and does not satisfy the fixed, $C_*$ construction.
\begin{figure}[t]
\centering
\includegraphics[width=\textwidth]{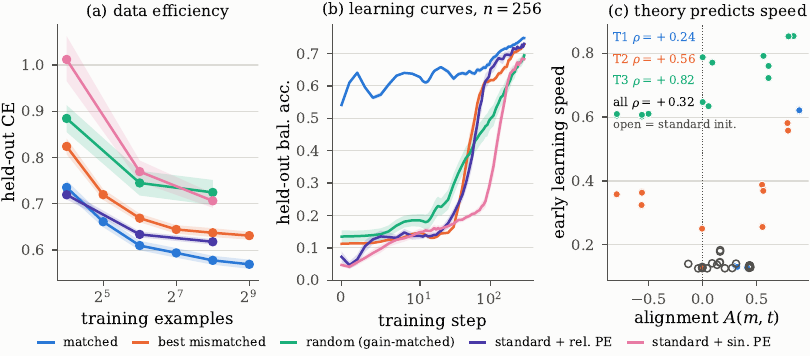}
\caption{Learning speed and data efficiency on T1: held-out cross-entropy versus training-set size (left) and learning curves (right) for the matched, mismatched and gain-matched random initializations.}
\label{fig:efficiency}
\end{figure}

\section{Vision experiments}
\label{app:vision}
We next test whether the same-dispersion relation based design principle applies in a less controlled image-classification setting. The goal is to alter the initialization of the 2-dimensional dispersion landscape of a ConViT-style attention block and test whether favoring heterogeneous spatial modes changes optimization on CIFAR-10. 

\subsection{Weight-tied ConViT architecture}
To stay in line with our theory, we use a weight-tied ConViT, in which one Transformer block is applied repeatedly across depth. A $2 \times 2$ patch embedding maps a $32 \times 32$ image to a $16 \times 16$ token grid ($N=256$). The embedding dimension is $d=216$, we employ $H=9$ heads and have depth $L=12$. Unless otherwise stated, we use a gated-positional self-attention (GPSA) as described in \citep{dascoli2021convit}, and a periodic relative encoding to stay firmly in line with our theory. The architecture has no class token and no absolute positional embedding. Classification uses mean pooling followed by LayerNorm and a linear layer. 

\subsection{Initialization}
We compare seven initializations. 

\begin{itemize}
    \item \textbf{default}: uses standard small random weights and disables structured positional locality. 
    \item \textbf{+I, -I}: these variants use a Mimetic-style \citep{trockman2023mimetic} initialization for $QK$ and initialize the full $OV$ operator, from a random matrix plus a positive or negative identity-biased component. It has the form $M_{OV} =0.4Z \pm 0.4I$
    \item \textbf{+pos}: these variants additionally initialize the GPSA positional logits with a localized ConViT positional encoding. Variants such as $-I+pos$ combine negative mimetic with a localized positional kernel. 
\end{itemize}

These initialization schemes allow us to separate the three ingredients suggested by the theory, namely the interplay of $OV$ geometry and attention positional structure in selecting relevant heterogeneous structures.

\subsection{Two-dimensional dispersion relation}

For a homogeneous embedded token state $x^*$, our architecture allows exact computation of the matrix dispersion relation. For each positional frequency $q=(q_1,q_2)$, we have:

\begin{equation}
    J(\mathbf{q})=C_* \lbrack I+\epsilon_a\sum_{h=1}^H\lambda_h(\mathbf{q})M_h\rbrack
\end{equation}

where $\epsilon_a$ is the learned attention residual scale. In this particular case, and given our initialization this is analytically tractable. Figure~\ref{fig:dispersion_vision} visualizes this $2$d dispersion for various initialization-positional structure combination. As predicted by the theory, without positional encoding the dispersion is nearly flat or dominated by the homogeneous mode. Local positional encoding makes the dispersion relation frequency dependent. We see the interplay with $OV$ geometry, negative-identity biased $OV$ amplifies heterogeneous segments while positive-identity biased $OV$ amplifies the homogeneous mode. We believe this helps rationalize the empirical observations that led to this initialization in the original work \cite{trockman2023mimetic} 

\begin{figure}[t]
\centering
\includegraphics[width=\textwidth]{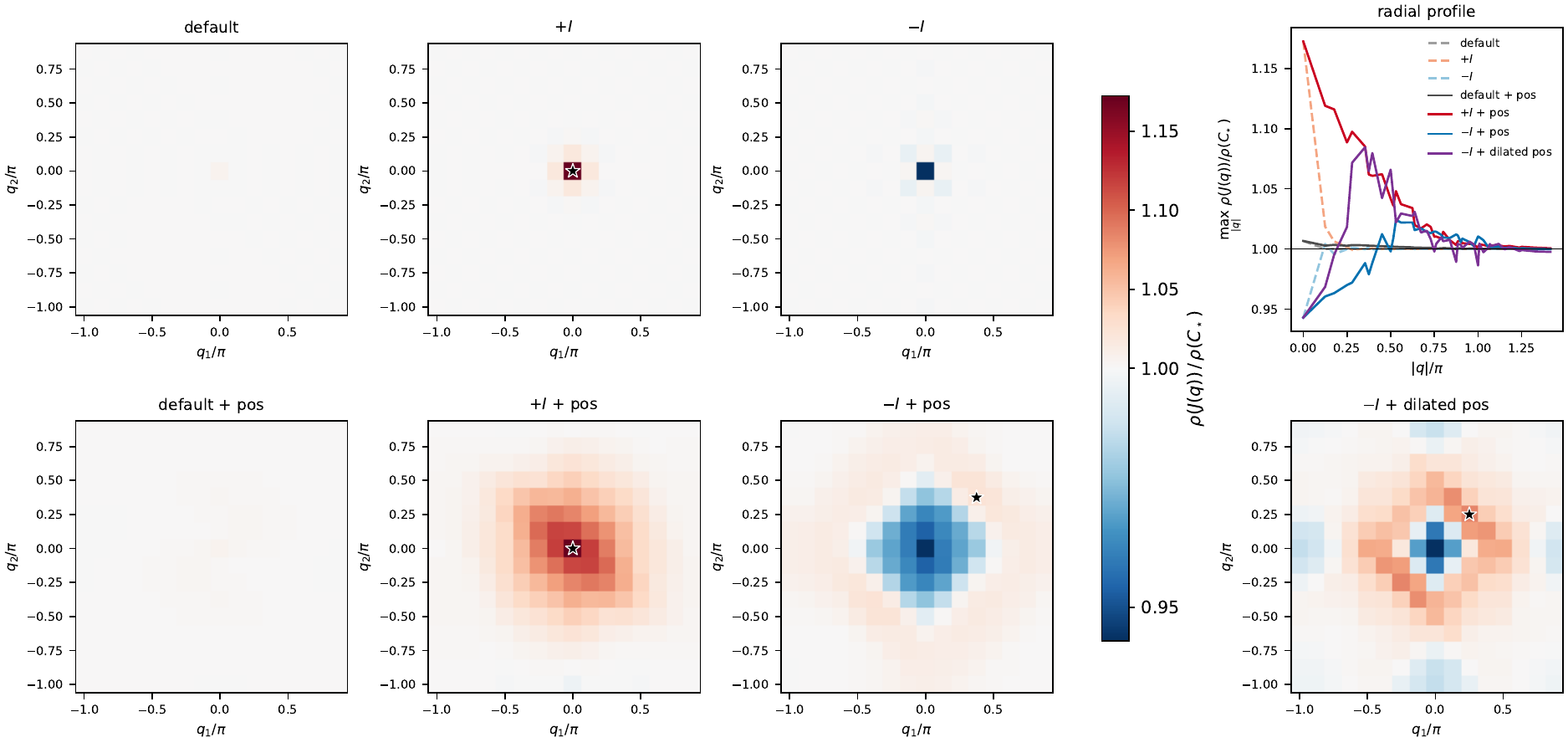}
\caption{Visualizations of the 2D dispersion relation for various combinations of positional structure and $OV$ geometry}
\label{fig:convit_dispersion}
\end{figure}

\subsubsection{Finite-depth validation of the dispersion prediction}
\label{app:vision_growth_validation}

Although the dispersion relation is a one-step local quantity, the tied network applies the same block $L=12$ times. We therefore test whether it predicts finite-depth perturbation growth. For each Fourier sector $\mathbf q$, we perturb the homogeneous state by $10^{-6}$ along the leading eigenvector of $J_0(\mathbf q)$ and measure the perturbation norm after 12 tied steps. We compare it with the exact linear prediction
\begin{equation}
w_L(\mathbf q)=J^{(L-1)}(\mathbf q)\cdots J^{(0)}(\mathbf q)v_0(\mathbf q),
\end{equation}
and the initialization-time proxy $\rho(J^{(0)}(\mathbf q))^L$.

As shown in Fig.~\ref{fig:growth}, the measured growth agrees with the exact linear evolution to below $5\times10^{-9}$ relative error for $+I+$pos, $-I+$pos, and engineered initialization. Moreover, $\rho(J_0(\mathbf q))^L$ closely predicts the relative growth across Fourier sectors, with Pearson correlations $0.9984$, $0.9993$, and $0.9999$, respectively. Thus, the initialization-time dispersion map accurately predicts which spatial modes are preferentially amplified through the full tied depth in the local regime.

\begin{figure}
\centering
\includegraphics[width=\textwidth]{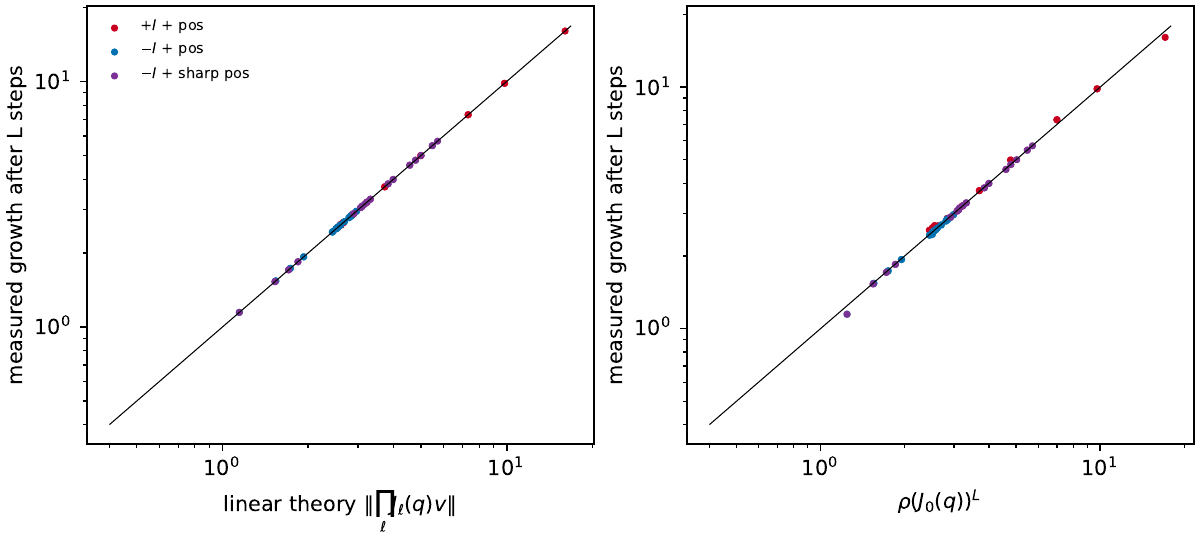}
\caption{Comparison of linear theory growth prediction vs actual deviation from homogeneous state from various positional encoding and $OV$ geometry combinations}
\label{fig:growth}
\end{figure}

\subsection{CIFAR-10 training protocol}

We use CIFAR-10 with a fixed class-balanced $45$k/$5$k train/validation split and the standard $10$k test set. Training uses random resized crops, horizontal flips, RandAugment, color jitter, random erasing, and standard normalization; full-data results are averaged over three seeds.

Models are trained for 100 full-data-equivalent epochs with AdamW (batch size $512$, learning rate $3\times10^{-3}$, weight decay $10^{-2}$), five epochs of warmup followed by cosine decay, and gradient clipping at norm $1$. Biases, positional parameters, GPSA gates, and residual scales are exempt from weight decay. For data-efficiency experiments, we keep the number of optimization steps fixed to the full-data schedule across dataset fractions.

We report final test accuracy, test-accuracy AUC, and the first epoch $t_{90}$ reaching $90\%$ test accuracy; seeds that do not reach the threshold are excluded from $t_{90}$ and the number of successful seeds is reported.

\subsubsection{Optimization results}

Figure~\ref{fig:app_vision_curves} summarizes the full-data results. $OV$ sign alone has little effect without positional structure, whereas combining localized attention with negative $OV$ geometry substantially accelerates optimization. Sharpening the positional kernel further improves learning speed, with a smaller improvement in final accuracy.

\begin{figure}[t]
\centering
\includegraphics[width=\textwidth]{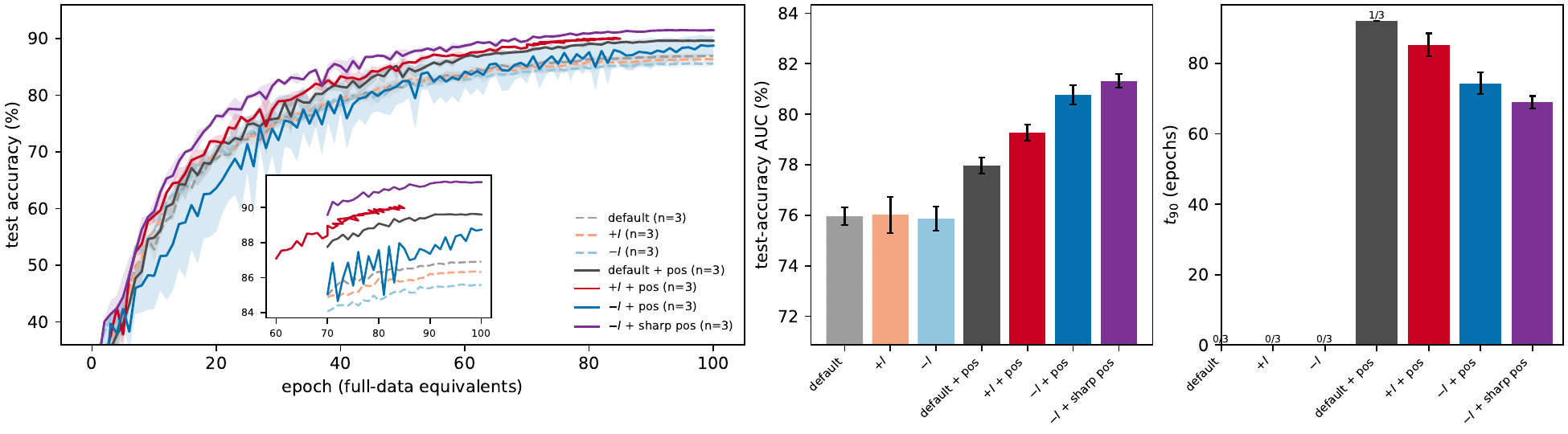}
\caption{\textbf{CIFAR-10 optimization under dispersion-aware initialization.}
Test accuracy (left), test-accuracy AUC (middle), and epoch reaching $90\%$
test accuracy (right), over three seeds. Localized positional attention combined
with negative $OV$ geometry accelerates training, with the engineered
$-I+$ sharp-pos initialization performing best overall.}
\label{fig:app_vision_curves}
\end{figure}

\subsubsection{Data efficiency}
\label{app:vision_data_efficiency}

We repeat the comparison using $5\%$--$100\%$ of the training set while keeping the optimization-step budget fixed. As shown in Fig.~\ref{fig:app_vision_data_efficiency}, the engineered initialization has the highest mean final accuracy and AUC at every evaluated fraction, with the largest gains over the unstructured default in the low- and intermediate-data regimes.

\begin{figure}[t]
\centering
\includegraphics[width=0.82\textwidth]{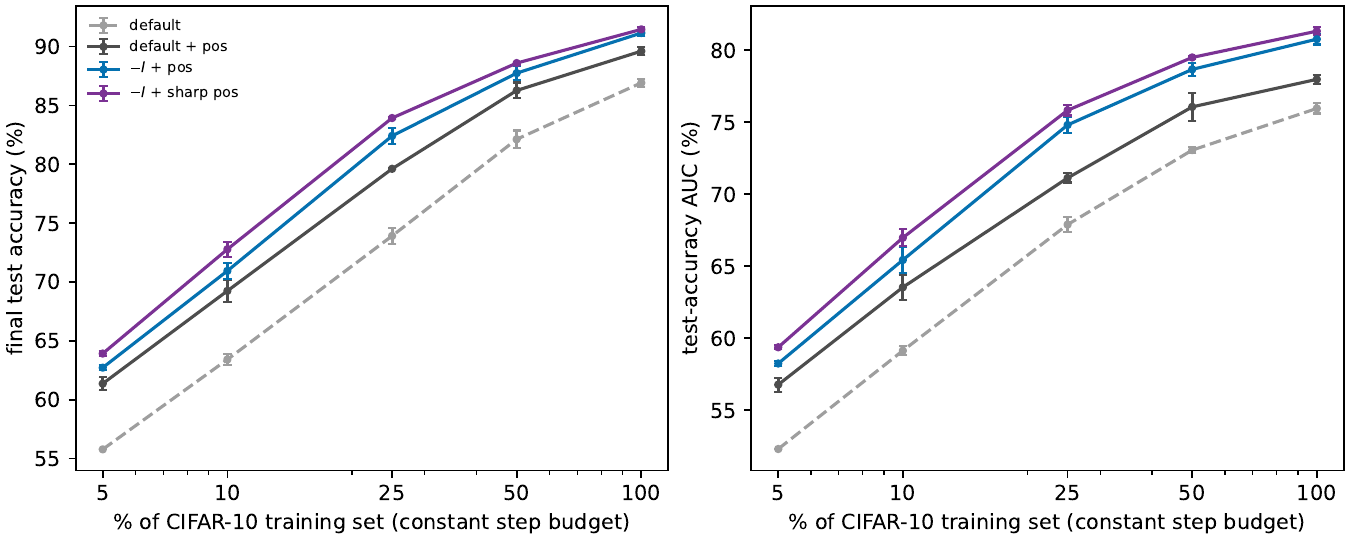}
\caption{\textbf{Data efficiency on CIFAR-10.}
Final test accuracy (left) and test-accuracy AUC (right) for different training-set
fractions. Error bars show one standard deviation across three seeds, except the
$5\%$ default baseline ($n=2$). The engineered initialization remains strongest
across the sweep.}
\label{fig:app_vision_data_efficiency}
\end{figure}

\subsubsection{Training diagnostics and dispersion evolution}

Figure~\ref{fig:app_vision_diagnostics} shows that the positional variants maintain heterogeneous token representations while both the attention residual scale and GPSA positional gate adapt during training. Thus, the engineered initialization biases the initial dynamics without freezing a prescribed attention operator.

\begin{figure}[t]
\centering
\includegraphics[width=\textwidth]{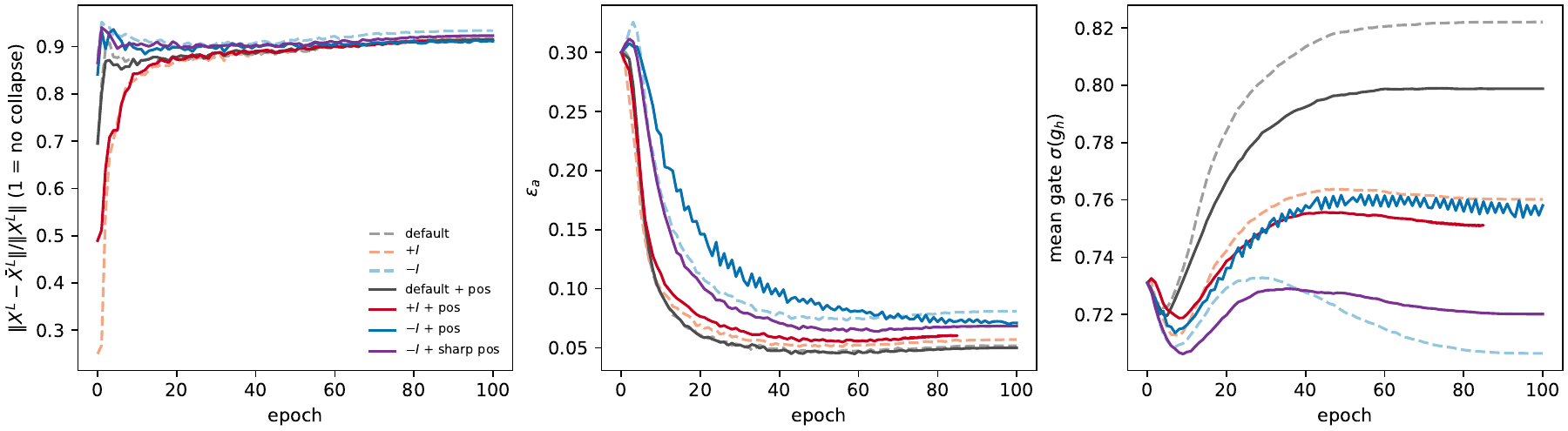}
\caption{\textbf{Training diagnostics for the tied ConViT.}
Final-depth token heterogeneity (left), learned attention residual scale
$\epsilon_a$ (middle), and mean GPSA positional gate (right). Representations
remain non-collapsed while the attention parameters adapt throughout training.}
\label{fig:app_vision_diagnostics}
\end{figure}

We also recompute the two-dimensional dispersion relation from checkpoints
(Fig.~\ref{fig:app_vision_evolution}). Holding negative $OV$ geometry fixed, the
sharper positional kernel produces a stronger initial heterogeneous-mode bias, but
both spectra subsequently reorganize. Dispersion engineering therefore acts as an
initial preference rather than a persistent spectral constraint.

\begin{figure}[t]
\centering
\begin{minipage}[t]{0.49\textwidth}
\centering
\includegraphics[width=\textwidth]{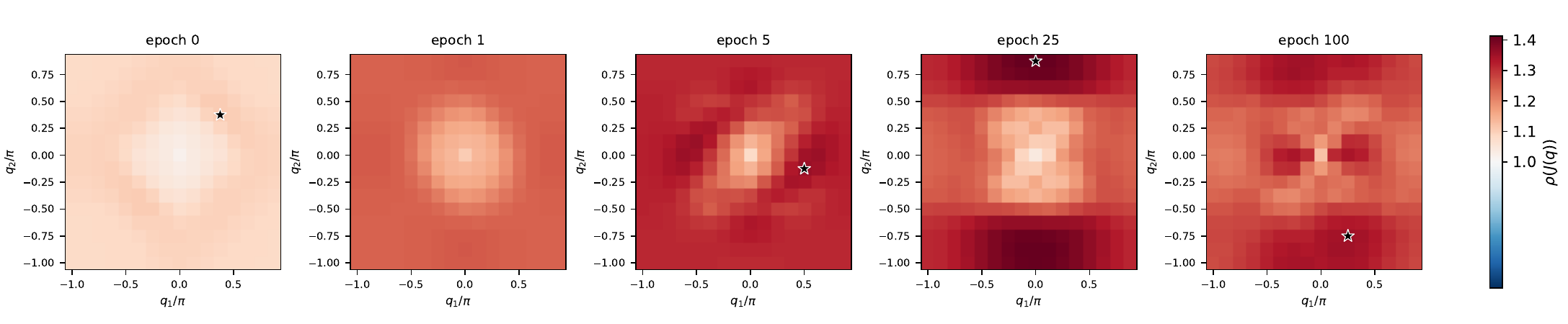}
\vspace{-0.5em}
\centerline{\small $-I+$pos}
\end{minipage}\hfill
\begin{minipage}[t]{0.49\textwidth}
\centering
\includegraphics[width=\textwidth]{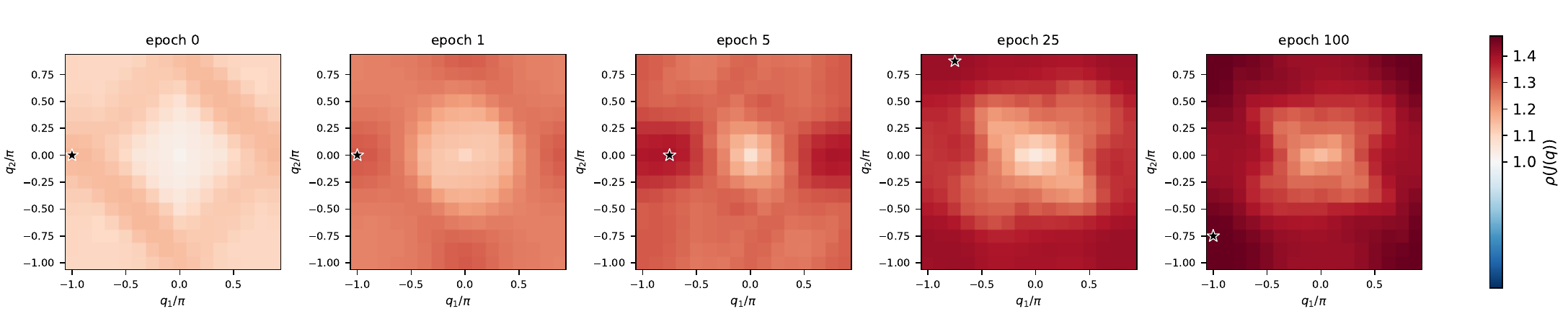}
\vspace{-0.5em}
\centerline{\small engineered $-I+$sharp pos}
\end{minipage}
\caption{\textbf{Evolution of the spatial dispersion relation.}
$\rho(J(\mathbf q))$ at epochs $0,1,5,25,$ and $100$ for one seed; stars denote
the maximally amplified Fourier sector. Both initial spectra evolve substantially
during training.}
\label{fig:app_vision_evolution}
\end{figure}

\subsubsection{Robustness and scope}
\label{app:vision_ablations}

Table~\ref{tab:app_vision_ablations} summarizes two architectural ablations. The
benefit of combining positional locality with negative $OV$ geometry persists with
pre-LN, although the additional gain from positional sharpening does not. In
contrast, injecting the positional bias inside the content softmax (PSA) is markedly
less stable at the same initialization scale, with several runs diverging. Thus, the
effect is robust to normalization but depends on how positional structure enters
attention.

\begin{table}[t]
    \centering
    \small
    \setlength{\tabcolsep}{7pt}
    \renewcommand{\arraystretch}{1.12}
    \begin{tabular}{@{}lccc@{}}
        \toprule
        Initialization
        & Final acc. (\%)
        & AUC (\%)
        & $t_{90}$ \\
        \midrule

        \multicolumn{4}{@{}l}{\textit{GPSA, no normalization}} \\[1pt]
        default + pos
        & $89.6 \pm 0.4$
        & $78.0 \pm 0.3$
        & $92.0\ (1/3)$ \\

        $-I$ + pos
        & $91.1 \pm 0.3$
        & $80.8 \pm 0.4$
        & $74.3 \pm 3.1$ \\

        $-I$ + sharp pos
        & $\mathbf{91.5 \pm 0.2}$
        & $\mathbf{81.3 \pm 0.3}$
        & $\mathbf{69.0 \pm 1.7}$ \\

        \addlinespace[2pt]
        \multicolumn{4}{@{}l}{\textit{GPSA, pre-LN}} \\[1pt]
        default + pos
        & $88.5 \pm 0.4$
        & $77.0 \pm 0.4$
        & -- \\

        $-I$ + pos
        & $\mathbf{89.8 \pm 0.3}$
        & $\mathbf{79.5 \pm 0.4}$
        & $88.0\ (1/3)$ \\

        $-I$ + sharp pos
        & $89.5 \pm 0.2$
        & $78.5 \pm 0.2$
        & -- \\

        \addlinespace[2pt]
        \multicolumn{4}{@{}l}{\textit{PSA, no normalization}} \\[1pt]
        $-I$ + pos
        & $87.3 \pm 0.1\ (2/3)$
        & $75.2 \pm 0.2\ (2/3)$
        & -- \\

        $-I$ + sharp pos
        & $86.8\ (1/3)$
        & $73.4\ (1/3)$
        & -- \\

        \bottomrule
    \end{tabular}

    \caption{\textbf{Normalization and attention-form ablations.}
    Parentheses indicate the number of contributing seeds when fewer than
    three runs are available; for $t_{90}$, they indicate the number of seeds
    reaching $90\%$ accuracy within 100 epochs.}
    \label{tab:app_vision_ablations}
\end{table}

\section{A pattern-formation primer: the Swift--Hohenberg model}
\label{app:swift}
This appendix illustrates, on the canonical pattern-forming Swift--Hohenberg model~\citep{cross2009pattern}, the two objects used in the main text: the dispersion relation and the amplitude equation. As in Eq.~\ref{eq:lie_trotter}, the model couples a spatial propagation operator to local pointwise dynamics:
\begin{equation}
\frac{\partial X}{\partial t}
=
\underbrace{-\left(q_0^2+\Delta\right)^2X}_{\mathcal P(X):\,\text{propagation}}
+
\underbrace{\mu X+\phi(X)}_{\mathcal L(X):\,\text{local dynamics}},
\label{eq:swift_hohenberg}
\end{equation}
where $\phi(0)=0$ and
$\phi(X)=\alpha_0X+\alpha_1X^2+\alpha_2X^3+\mathcal O(X^4)$.
The homogeneous state $X=0$ is reminiscent of the collapsed Transformer state. To probe its stability, consider a perturbation
$U\propto \euler^{\sigma(k)t+\im k\cdot x}$, where $k$ encodes spatial scale and $\sigma(k)$ its growth rate. Linearizing Eq.~\ref{eq:swift_hohenberg} gives the dispersion relation
\begin{equation}
\sigma(k)
=
\mu+\alpha_0-\left(q_0^2-|k|^2\right)^2.
\label{eq:swift_dispersion}
\end{equation}
Its sign determines whether a perturbation grows or decays, while its shape determines which mode is preferentially amplified: here the critical mode is $q_c=q_0$. Thus, the dispersion relation simultaneously characterizes the loss of stability of the homogeneous state and the structure that emerges from it, as illustrated in Fig.~\ref{fig:swift}.

\begin{figure}[t]
\centering
\includegraphics[width=0.9\textwidth]{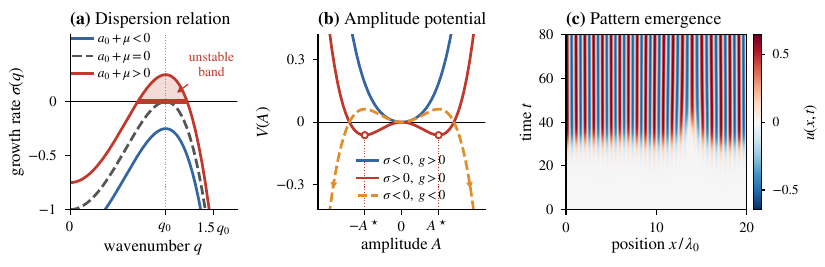}
\caption{The Swift-Hohenberg illustration: the dispersion relation (a) defines the emergent perturbations, the amplitude potential in (b) dictates their stability, leading to the apparition of patterns in (c).}
\label{fig:swift}
\end{figure}

Linear analysis, however, does not determine the fate of the emerging mode. Near onset, the solution is dominated by
$U\approx B(t)\euler^{\im q_cx}$, and nonlinear effects govern whether its amplitude diverges or saturates. Expanding Eq.~\ref{eq:swift_hohenberg} to the third order and projecting onto the critical mode yields reduced amplitude dynamics of the form
\begin{equation}
\frac{dB}{dt}
=
-\frac{\partial \mathcal V}{\partial B},
\qquad
\mathcal V(B)
=
-\frac{\sigma}{2} B^2
+
\frac{g(\alpha_1,\alpha_2)}{4}B^4.
\label{eq:amplitude_potential}
\end{equation}
The dispersion relation controls the stability of the homogeneous state, while the nonlinear coefficient $g(\alpha_1,\alpha_2)$ controls the stabilization of the emerging structure; for $g>0$, the growing mode saturates at finite amplitude. We see in Fig.~\ref{fig:swift}(b) the impact of $g$, which can be viewed as engineering the stability landscape of the emergent pattern. 

This example contains the two ingredients used in the main text: the dispersion relation determines when the homogeneous state loses stability and which structures become accessible, while nonlinearities determine their subsequent fate. Crucially, in this example the propagation operator $\mathcal P(X)$ diagonalizes along Fourier modes, reducing the stability analysis to independent spatial scales. In Transformers, positional encoding and multi-head attention play the analogous role (Theorem~\ref{thm:TMDR}), giving rise to a matrix-valued dispersion relation.

\end{document}